\ifdefined\journalstyle
  \documentclass{svproc}

\else
  \documentclass[11pt]{article} % plain/arXiv formatting
  \def\arxivstyle{1}
  \usepackage{theorem}
  \newtheorem{theorem}{Theorem}%[section]       
  \newtheorem{lemma}%[theorem]
  {Lemma}
  \newenvironment{proof}{\begin{trivlist}
  \item[\hskip\labelsep{\it Proof.}]}{\end{trivlist}}
  \newcommand{\qed}{$\Box$}
  
  \usepackage{pdfsync}
  \usepackage[hidelinks]{hyperref}

  \usepackage{latexsym,amsfonts,amsmath,amssymb,url}
  \usepackage{graphicx}
  \usepackage{subcaption} % for subfigures

  \usepackage[usenames,dvipsnames,svgnames]{xcolor}

 \fi

\usepackage{type1cm}        % activate if the above 3 fonts are
\usepackage{makeidx}         % allows index generation
\usepackage{graphicx}        % standard LaTeX graphics tool
\graphicspath{{KKNS-figures/}}

\usepackage{multicol}        % used for the two-column index
\usepackage[bottom]{footmisc}% places footnotes at page bottom

\usepackage{newtxtext}       %
\usepackage[varvw]{newtxmath}       % selects Times Roman as basic font

\usepackage{siunitx}% for table alignments

\newcounter{algorithm}%[section]

\newcommand{\bsh}{{\boldsymbol{h}}}

\newcommand{\bsm}{{\boldsymbol{m}}}

\newcommand{\bst}{{\boldsymbol{t}}}

\newcommand{\bsv}{{\boldsymbol{v}}}

\newcommand{\bsy}{{\boldsymbol{y}}}
\newcommand{\bsz}{{\boldsymbol{z}}}

\newcommand{\bszero}{{\boldsymbol{0}}} % vector of zeros
\newcommand{\bsone}{{\boldsymbol{1}}}  % vector of ones
\newcommand{\bsalpha}{{\boldsymbol{\alpha}}}

\newcommand{\bsgamma}{{\boldsymbol{\gamma}}}

\newcommand{\bsnu}{{\boldsymbol{\nu}}}

\newcommand{\bsDelta}{{\boldsymbol{\Delta}}}

\newcommand{\rd}{{\mathrm{d}}}
\newcommand{\ri}{{\mathrm{i}}}

\newcommand{\bbE}{{\mathbb{E}}}

\newcommand{\bbN}{{\mathbb{N}}}

\newcommand{\bbR}{{\mathbb{R}}}

\newcommand{\bbZ}{{\mathbb{Z}}}
\DeclareSymbolFont{bbold}{U}{bbold}{m}{n}
\DeclareSymbolFontAlphabet{\mathbbold}{bbold}

\newcommand{\calA}{{\mathcal{A}}}

\newcommand{\calE}{{\mathcal{E}}}

\newcommand{\calI}{{\mathcal{I}}}
\newcommand{\calJ}{{\mathcal{J}}}
\newcommand{\calK}{{\mathcal{K}}}

\newcommand{\calO}{{\mathcal{O}}}

\newcommand{\calR}{{\mathcal{R}}}
\newcommand{\calS}{{\mathcal{S}}}

\newcommand{\calW}{{\mathcal{W}}}

\newcommand{\setu}{{\mathfrak{u}}}
\newcommand{\setv}{{\mathfrak{v}}}

\providecommand{\argmin}{\operatorname*{argmin}}

\begin{document}

\ifdefined\journalstyle

\mainmatter              % start of a contribution
\title{Lattice-based Deep Neural Networks:\\ Regularity and Tailored Regularization}
\titlerunning{Lattice-based Deep Neural Networks}  % abbreviated title (for running head)
\author{Alexander Keller\inst{1} \and Frances Y. Kuo\inst{2} \and
Dirk Nuyens\inst{3} \and Ian H. Sloan\inst{2}}
%
%\authorrunning{Alexander Keller et al.} % abbreviated author list (for running head)
\authorrunning{Keller, Kuo, Nuyens, Sloan} % abbreviated author list (for running head)
\tocauthor{Alexander Keller, Frances Y. Kuo, Dirk Nuyens, and Ian H. Sloan}
\institute{NVIDIA, Germany,\\
 \email{akeller@nvidia.com} \\ 
 %WWW home page:
 %\texttt{http://users/\homedir iekeland/web/welcome.html}
 \and
 School of Mathematics and Statistics, UNSW Sydney, \\
 \email{f.kuo@unsw.edu.au}, \email{i.sloan@unsw.edu.au} \\
  \and
 Department of Computer Science, KU Leuven, Belgium, \\
 \email{dirk.nuyens@kuleuven.be}
}
\else
  \title{Lattice-based Deep Neural Networks:\\ Regularity and Tailored Regularization}
  \author{Alexander Keller, Frances Y. Kuo, Dirk Nuyens, and Ian H. Sloan}
  \date{March 2026}
\fi

\maketitle              % typeset the title of the contribution

\begin{abstract}
This survey article is concerned with the application of lattice rules to Deep Neural Networks (DNNs), 
lattice rules being a family of quasi-Monte Carlo methods. They have demonstrated effectiveness in various contexts for high-dimensional integration and function approximation.
They are extremely easy to implement thanks to their very simple formulation --- all that is required is a good integer generating vector of length matching the dimensionality of the problem. 
In recent years there has been a burst of research activities on the application and theory of DNNs. We review our recent article on using lattice rules as training points for DNNs with a smooth activation function, where we obtained explicit regularity bounds of the DNNs. By imposing restrictions on the network parameters to match the regularity features of the target function, we prove that DNNs with tailored lattice training points can achieve good theoretical generalization error bounds, with implied constants independent of the input dimension. We also demonstrate numerically that DNNs trained with our tailored regularization perform significantly better than with standard $\ell_2$ regularization.

\ifdefined\journalstyle
\else
 \bigskip\noindent\textbf{keywords:} Lattice rules, lattice algorithms, deep neural networks, tailored lattice training points, tailored regularization.
\fi
\end{abstract}
\section{Introduction}

This survey article reviews the application of so-called \emph{lattice rules} to \emph{Deep Neural Networks} (DNNs) \cite{GBC16,LBH15}.
Lattice rules are a family of \emph{quasi-Monte Carlo} (QMC) methods, which are equal-weight cubature rules for approximating integrals over the $s$-dimensional unit cube
\begin{align} \label{KKNS:eq:qmc}
  I(F) \,:=\, \int_{[0,1]^s} F(\bsy)\,\rd\bsy
  \quad\approx\quad
  Q(F) \,:=\, \frac{1}{N} \sum_{k=1}^N F(\bst_k),
\end{align}
where the points $\bst_1,\ldots,\bst_N \in [0,1]^s$ are designed to be more uniformly distributed than random points, with the aim of improving the slow $\calO(N^{-1/2})$ convergence of simple Monte Carlo methods.
A large community of QMC researchers actively contributes to the field. Monographs and review articles in this domain are \cite{DKP22,DKS13,DP10,Hic98,KS05,Lem09,LP14,Nie92,SJ94}. This survey article has been prepared as part of the proceedings for the 16th biennial conference dedicated to advancing research in this area.

We (the four authors of this article) are particularly fond of lattice rules, due to their simplicity of implementation, as well as a group structure that facilitates error analysis.
Lattice rules (and more generally, QMC methods) have been very successful in a range of applications. Selected contributions by the present authors span computer graphics (e.g., \cite{ImgSynthR1Lattice,SimOnR1L,Myths04,IEQandML,IndustryQMC}), uncertainty quantification for partial differential equations (PDEs) with random inputs (e.g., \cite{DKLNS14,GGKSS19,GKNSSS15,GKNSS11,GKNSS18b,GKKSS24,HHKKS24,KKKNS22,KKS20,KKS24,KN16,KSS12}), and mathematical finance (e.g., \cite{ACN13a,ACN13b,GKSS24,WS05,WS07,WS11}), where quasi-Monte Carlo techniques were initially explored in foundational works \cite{ABG98,CMO97,PT95}.

In this survey article, we summarize findings from our recent work \cite{KKNS25a}. There, we interpret DNNs as \emph{nonlinear algorithms for function approximation} and employ \emph{lattice points as training points}. These lattice-based DNNs are thus positioned as alternatives to the approximation algorithms such as truncated trigonometric series, multivariate polynomial and kernel methods. A natural question arises regarding the comparative performance of these methods. We provide partial theoretical insights and encourage further exploration in this area. Other applications of QMC methods in machine learning \cite{GK22,MT18,SB14} include works such as \cite{DF21,FW23,GHN25,KV23,LMR20,LMRS21,MR21,XFW24}, especially \cite{BWM18,Liu23,LO21} on variational inference.

The structure of this article is as follows. In Section~\ref{KKNS:sec:qmc} we give a high level summary of three theoretical settings for lattice rules used for integration. We additionally highlight related results on both integration and approximation, which were not covered in \cite{KKNS25a}. In Section~\ref{KKNS:sec:dnn} we first introduce the DNNs and explain that for the error analysis we need to have regularity bounds on the DNNs. We summarize three key theorems from \cite{KKNS25a}: Theorem~\ref{KKNS:thm:reg} provides the regularity bounds of the DNNs in terms of the ``network parameters'' and derivative bounds on the ``activation function''; Theorem~\ref{KKNS:thm:diff} shows that we can restrict the network parameters of the DNNs to match the regularity features of the target function in order to control the norms which affect the ``generalization gap''; Theorem~\ref{KKNS:thm:err} states that we can construct lattice generating vectors to achieve certain error bounds independently of the input dimension $s$. We also present a tailored regularization approach from \cite{KKNS25a} which was shown numerically in \cite{KKNS25a} to perform better than the standard $\ell_2$ regularization. In Section~\ref{KKNS:sec:num} we carry out numerical experiments, and in Section~\ref{KKNS:sec:conc} we give concluding remarks.

While this is mostly a survey article, there are some new contributions. In Lemma~\ref{KKNS:lem:common} we generalize the derivative bounds of the activation functions from \cite{KKNS25a}. In particular, we consider the generalized ``swish'' activation function $x/(1+e^{-cx})$ which converges pointwise to the ``ReLU'' activation function $\max\{x,0\}$ when $c\to\infty$. In Lemma~\ref{KKNS:lem:lower} we show that the factorial growth in the derivative bounds obtained in \cite{KKNS25a} and in Lemma~\ref{KKNS:lem:common} cannot be avoided, by providing a lower bound. The numerical experiments in Section~\ref{KKNS:sec:num} for the generalized swish function with varying parameters are new. In the Appendix we provide full details of the proof of Theorem~\ref{KKNS:thm:err} which was omitted from \cite{KKNS25a}; these make use of the review material in Section~\ref{KKNS:sec:qmc}.

\section{Review of lattice theory for integration and approximation} \label{KKNS:sec:qmc}

The $N$ points of a (rank-$1$) lattice rule are given by
\begin{align} \label{KKNS:eq:lat}
 \bst_k := \left\{\frac{k\bsz}{N}\right\} = \frac{k\bsz\bmod N}{N}, \quad k=1,\ldots,N,
\end{align}
where $\bsz\in\bbZ^s$ is the \emph{generating vector}, and 
the braces around a vector indicate that we take the fractional part
of each component in the vector. Equivalently, the points can be indexed from $0$ to $N-1$, with $\bst_0\equiv\bst_N$.
The quality of a lattice rule is determined by the choice of its generating vector $\bsz$. 
We may restrict each component $z_j$ to the set $\bbZ_N := \{z\in\bbZ : 1\le z\le N-1 \mbox{ and } \gcd(z,N)=1\}$, which ensures that every one-dimensional projection of the lattice rule is exactly the equally-spaced points $\{k/N: k=0,1,\ldots,N-1\}$. Note that repeated components between dimensions may be bad as they lead to a bad lower-dimensional projection. Figure~\ref{KKNS:fig:lat} (left)
illustrates a $64$-point lattice rule in two dimensions ($s=2$).

\begin{figure}[t]
 \centering
 \includegraphics[width = 12cm]{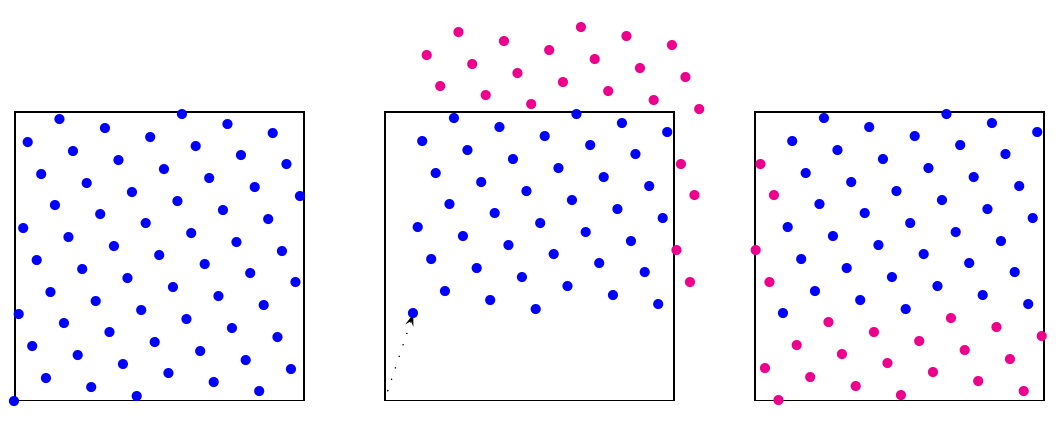}
 \caption{Applying a $(0.1,0.3)$-shift to a $64$-point lattice rule in two dimensions.
 Left: original lattice rule. Middle: moving all points by $(0.1,0.3)$.
 Right: wrapping the points back inside the unit cube.}
 \label{KKNS:fig:lat}
\end{figure}

A \emph{randomly shifted lattice rule} (or more generally, a randomly shifted QMC method) takes the form
\begin{align*}
  Q(F,\bsDelta) \,:=\, \frac{1}{N} \sum_{k=1}^N F(\{\bst_k+\bsDelta\}),
\end{align*}
where $\bsDelta \in [0,1]^s$ is the \emph{shift} with independent components uniformly generated from $[0,1]$, and the braces mean taking the fractional parts as above. Figure~\ref{KKNS:fig:lat} (right) illustrates the result of shifting a $64$-point lattice rule in two dimensions by $\bsDelta=(0.1,0.3)$. Random shifting realizes an unbiased estimator to the integral, since $\bbE[Q(F,\cdot)] = I(F)$. Taking multiple (e.g., $16$) independent random shifts of the same deterministic rule provides a practical error estimate, see e.g., \cite{DKS13} and the numerical experiments in Section~\ref{KKNS:sec:num}.

Lattice points can also be used for function approximation. We refer to these as \emph{lattice-based algorithms}. Consider one-periodic real-valued $L_2$ functions defined on $[0,1]^s$
with absolutely convergent Fourier series
\begin{align*}
  F(\bsy) \,=\, \sum_{\bsh\in\bbZ^s} \hat{F}_\bsh\,e^{2\pi\ri\bsh\cdot\bsy},
  \quad\mbox{with}\quad
  \hat{F}_\bsh \,:=\, \int_{[0,1]^s} F(\bsy)\,e^{-2\pi\ri\bsh\cdot\bsy}\,\rd\bsy,
\end{align*}
where $\bsh\cdot\bsy = h_1y_1 + \cdots + h_sy_s$ denotes the usual dot product.
It is not possible to evaluate the sum for infinitely many indices $\bsh\in\bbZ^s$. Moreover, the Fourier coefficients $\hat{F}_\bsh$ are themselves integrals and need to be approximated. Thus, one lattice-based strategy to approximate $F$ is to first truncate the Fourier series to a well-chosen finite index set $\calA_s \subset \bbZ^s$ and then approximate the Fourier coefficients for $\bsh\in\calA_s$ by a lattice rule with $N$ points and generating vector $\bsz\in \bbZ_N^s$, giving the \emph{truncated trigonometric series} (see e.g., \cite{CKNS20,CKNS21,KMN25,KSW06,KWW09c} and \cite{BKUV17,Kam13,KPV15,PV16,KMNN21})
\begin{align} \label{KKNS:eq:trig}
  A^\dagger(F)(\bsy) \,:=\, \sum_{\bsh\in\calA_s} \hat{F}_\bsh^a \,e^{2\pi\ri\bsh\cdot\bsy},
  \quad\mbox{with}\quad
  \hat{F}_\bsh^a \,:=\, \frac{1}{N} \sum_{k=1}^N F(\bst_k)\,e^{-2\pi\ri \bsh\cdot\bst_k}.
\end{align}
Thanks to the structure of lattice points, the evaluation of $\hat{F}_\bsh^a$ for all $\bsh\in\calA_s$ can be computed in $\calO(N\,\log N)$ operations using a one-dimensional Fast Fourier Transform (FFT). The cost for evaluating $A^\dagger(F)$ then depends on the structure and size of the index set $\calA_s$.

Suppose that $F$ belongs to a reproducing kernel Hilbert space with kernel $K$. Then another lattice-based strategy to approximate $F$ is to use the \emph{kernel method} (see e.g., \cite{KKKNS22,KKS24,KMNSS24,ZLH06,ZKH09})
\begin{align} \label{KKNS:eq:ker}
  A^*(F)(\bsy) \,:=\, \sum_{k=1}^N a_k\,K(\bst_k,\bsy),
\end{align}
where the coefficients $a_k$ are obtained by solving the linear system
\begin{align} \label{KKNS:eq:sys}
  F(\bst_{k'}) \,=\, \sum_{k=1}^N a_k\,K(\bst_k,\bst_{k'}) \quad\mbox{for all } k'=1,\ldots,N,
\end{align}
which ensures that the kernel approximation interpolates $F$ at the lattice points.
If the reproducing kernel is shift-invariant, i.e., $K(\bsy,\bsy') = K(\{\bsy-\bsy'\},\bszero\}$, as is the case of Fourier series with trigonometric basis functions, then, together with the structure of lattice points, the linear system \eqref{KKNS:eq:sys} involves a circulant matrix and so the coefficients $a_k$ can be solved in $\calO(N\,\log N)$ operations using FFT. The cost for evaluating $A^*(F)$ then depends on the definition of the kernel $K$.

Just as in lattice rules for the integration problem, the quality of the lattice-based approximation algorithms $A^\dagger(F)$ and $A^*(F)$ depends on the choice of the lattice generating vector $\bsz$ (and in the case of $A^\dagger(F)$ it also depends on the choice of the index set $\calA_s$). In various settings for both integration and approximation, we know how to construct generating vectors \emph{component by component} to achieve good theoretical error bounds: inductively, the $j$th component $z_j\in\bbZ_N$ of the lattice generating vector $\bsz$ is chosen to minimize some error criterion or to satisfy a required property, with the previously chosen components $z_1,\ldots, z_{j-1}$ held fixed.
See the papers cited above and the references therein, most of them rely on the ``fast'' construction method from \cite{NC06a} that utilizes the structure of lattice points with FFT. We will review some of these settings in the next subsections.

\subsection{Lattice rules for integration} \label{KKNS:sec:lat-int}

Denoting the underlying function space for the integrands $F$ generically by $\calW$, we define the \emph{worst case error} of an $N$-point lattice rule \eqref{KKNS:eq:lat} with generating vector $\bsz$ (or more generally, a QMC method \eqref{KKNS:eq:qmc}) in this function space by
\[
  e^{\rm wor}_N(\bsz) := \sup_{F\in\calW,\, \|F\|_\calW \le 1} |I(F) - Q(F)|,
\]
where $\|\cdot\|_{\calW}$ denotes the norm of $\calW$. Then we have the generic integration error bound
\begin{align} \label{KKNS:eq:KH-ineq}
  |I(F) - Q(F)| \le e^{\rm wor}_N(\bsz)\, \|F\|_\calW \qquad\mbox{for all } F\in\calW.
\end{align}
Such an error bound separates the dependence on the integrand $F$ from the dependence on the generating vector $\bsz$. In classical QMC theory we have the famous Koksma--Hlawka inequality taking the same form~\eqref{KKNS:eq:KH-ineq}. 

Following \cite{DSWW06,SW98}, contemporary QMC settings work with \emph{weighted} function spaces. There is a weight parameter $\gamma_\setu>0$ associated with each subset of variables $\bsy_\setu := (y_j)_{j\in\setu}$ for $\setu\subseteq\{1:s\} := \{1,2,\ldots,s\}$, and together they moderate the relative importance of different subsets of variables, with a small value of $\gamma_\setu$ meaning that the function depends more weakly on $\bsy_\setu$. These weights $\bsgamma := (\gamma_\setu)_{\setu\subseteq\{1:s\}}$ not only appear in the definition of the norm $\|\cdot\|_{\calW}$ but also in the formula for the worst case error $e^{\rm wor}_N(\bsz)$. With given weights, the strategy is to choose the lattice generating vector~$\bsz$ to minimize the worst case error $e^{\rm wor}_N(\bsz)$ as much as possible by a
component-by-component construction, as briefly discussed earlier in this section. The resulting error bound is naturally proved by induction. Sometimes we consider randomly shifted lattice rules and then the
error bounds hold in a root-mean-square (r.m.s.) sense with respect to the random
shift. 

We summarize three theoretical settings for lattice rules which are needed later in this survey. In the following, we write $\partial^\alpha_\setu:= \prod_{j\in\setu} (\frac{\partial}{\partial y_j})^\alpha$, and
\[
  \varrho_\alpha(\lambda) := 
  \frac{2\zeta(2\alpha\lambda)}{(2\pi)^{2\alpha\lambda}}, \qquad \lambda > \tfrac{1}{\alpha}.
\]
where $\zeta(x) := \sum_{h=1}^\infty h^{-x}$ is the Riemann zeta
function.

\begin{itemize}
\item [\textnormal{(a)}] \cite[Theorem~5.8]{DKS13} For a randomly shifted lattice rule in a
    (non-periodic, ``unanchored'') weighted Sobolev space
    ${\calW_{1,\bsgamma}}$ of smoothness one, with norm
\begin{align} \label{KKNS:eq:norm-np}
 \|F\|_{\calW_{1,\bsgamma}}^2
 := \sum_{\setu\subseteq\{1:s\}} \frac{1}{\gamma_\setu}
 \int_{[0,1]^{|\setu|}} \bigg| \int_{[0,1]^{s-|\setu|}}
 \partial_\setu^1 F(\bsy)\,
 \rd\bsy_{\{1:s\}\setminus\setu}\bigg|^2\rd\bsy_\setu,
\end{align}
a good generating vector $\bsz$ can be constructed such that 
\begin{align} \label{KKNS:eq:wce-np}
 {\rm r.m.s.}\; e^{\rm wor}_N(\bsz)
  \le \bigg(
  \frac{2}{N} \sum_{\emptyset\ne \setu \subseteq \{1:s\}} \gamma_\setu^\lambda\,
  \big[\varrho_1(\lambda)\,2^\lambda\big]^{|\setu|}\bigg)^{\frac{1}{2\lambda}}\,
  \quad\forall\;\lambda\in (\tfrac{1}{2},1].
\end{align}

\item [\textnormal{(b)}] \cite[Theorem~5.12]{DKS13} For a lattice rule in a (periodic, Hilbert)
    weighted Korobov space $\calW_{\alpha,\bsgamma}$ with integer
    smoothness parameter $\alpha \ge 1$, with norm
\begin{align} \label{KKNS:eq:norm-per}
 \|F\|_{\calW_{\alpha,\bsgamma}}^2
 := \sum_{\setu\subseteq\{1:s\}} 
 \frac{1}{\gamma_\setu}
 \int_{[0,1]^{|\setu|}} \bigg| \int_{[0,1]^{s-|\setu|}}
 \partial_\setu^\alpha F(\bsy)\,
 \rd\bsy_{\{1:s\}\setminus\setu}\bigg|^2\rd\bsy_\setu,
\end{align}
a good generating vector $\bsz$ can be constructed such that 
\begin{align} \label{KKNS:eq:wce-per}
  e^{\rm wor}_N(\bsz)
  \le \bigg(
  \frac{2}{N} \sum_{\emptyset\ne \setu\subseteq \{1:s\}}
  \gamma_\setu^\lambda\,
  \big[\varrho_\alpha(\lambda)\big]^{|\setu|}
  \bigg)^{\frac{1}{2\lambda}}\, \quad\forall\;\lambda\in
  (\tfrac{1}{2\alpha},1].
\end{align}

\item [\textnormal{(c)}] \cite[Lemma~3.1]{KKS20} For a lattice rule in a (periodic,
    non-Hilbert) weighted Korobov space $\calK_{\alpha,\bsgamma}$ with
    integer smoothness parameter $\alpha \ge 2$, with norm bound
\begin{align} \label{KKNS:eq:norm-K}
 \|F\|_{\calK_{\alpha,\bsgamma}}
 \le \max_{\setu\subseteq\{1:s\}} 
 \frac{1}{\gamma_\setu}
 \int_{[0,1]^{|\setu|}} \bigg| \int_{[0,1]^{s-|\setu|}}
 \partial_\setu^\alpha F(\bsy)\,
 \rd\bsy_{\{1:s\}\setminus\setu}\bigg| \,\rd\bsy_\setu,
\end{align}
a good generating vector $\bsz$ can be constructed such that 
\begin{align} \label{KKNS:eq:wce-K}
  e^{\rm wor}_N(\bsz)
  \le \bigg(
  \frac{2}{N} \sum_{\emptyset\ne \setu \subseteq \{1:s\}}
  \gamma_\setu^\lambda\,
  \big[\varrho_\alpha(\lambda/2)\big]^{|\setu|}
  \bigg)^{\frac{1}{\lambda}}\, \quad\forall\;\lambda\in
  (\tfrac{1}{\alpha},1].
\end{align}
\end{itemize}
All three results mentioned above were established for cases where
$N$ is a power of a prime, and analogous results hold for any
composite $N$. Additionally, it is possible to construct "embedded"
lattice rules \cite{CKN06} that are applicable across a range of values of $N$
(e.g., from $2^{10}$ to $2^{20}$), albeit with the trade-off
of introducing a small additional constant factor in the error bound.

We have intentionally presented the results in a way that ensures the
norm \eqref{KKNS:eq:norm-np} for the non-periodic case~(a) aligns with the
norm \eqref{KKNS:eq:norm-per} for the periodic case~(b) when $\alpha = 1$.
The Hilbert Korobov norm is typically defined in terms of the decay of
Fourier coefficients with a real parameter $\alpha > 1/2$.
Only when $\alpha$ is an integer can this norm be expressed in terms
of mixed derivatives as in \eqref{KKNS:eq:norm-per}. In some cases, the
parameter $\alpha$ is scaled by $1/2$, and frequently, the factor
$(2\pi)^{2\alpha|\setu|}$ is absorbed into the weights~$\gamma_\setu$.
Similarly, the non-Hilbert Korobov norm is typically defined via Fourier
series decay with a real parameter $\alpha > 1$, and only for integer
values of $\alpha$ can it satisfy a bound like \eqref{KKNS:eq:norm-K}, which
involves mixed derivatives. In classical lattice rule theory \cite{SJ94},
the corresponding worst case error is often referred to as $P_\alpha$.
Here, we use the notation $\calK_{\alpha,\bsgamma}$ to distinguish
this space from those associated with norms \eqref{KKNS:eq:norm-np} and
\eqref{KKNS:eq:norm-per}, avoiding potential confusion.

Setting~(a) achieves a convergence rate close to $\calO(N^{-1})$. In contrast,
both settings (b) and (c) can achieve a convergence rate close to
$\calO(N^{-\alpha})$, but the conditions imposed on the weights
$\gamma_\setu$ to ensure that the associated constants remain
independent of the input dimensionality $s$ can vary significantly.
In Theorem~\ref{KKNS:thm:err}, we will highlight a particular advantage of the non-Hilbert setting.

There are plenty of other theories for lattice rules, including tent-transformed lattice rules (see e.g., \cite{CKNS16,DNP14,Hic02}), lattice rules for the unbounded domain $\bbR^s$ (see e.g.,\cite{NK14,NS23}), lattice rules with a random number of points (see e.g., \cite{DGS22,KKNU19,KNW23}), and lattice rules based on a \emph{median of means} (see e.g., \cite{GL22}).

\subsection{Lattice-based algorithms for function approximation}

We switch now to the problem of function approximation. 
Denoting again the underlying function space for the target function $F$ generically by $\calW$, we define the \emph{worst case $L_2$ approximation error} of a lattice-based algorithm $A(F)$ in this function space by
\[
  e^{\rm wor-app}_N(A,\bsz) := \sup_{F\in\calW,\, \|F\|_\calW \le 1} \|F - A(F)\|_{L_2},
\]
where $A(F)$ could be e.g., the truncated trignometric series $A^\dagger(F)$ defined in \eqref{KKNS:eq:trig}, or the kernel method $A^*(F)$ defined in \eqref{KKNS:eq:ker}. Then we have the generic approximation error bound in the same form as \eqref{KKNS:eq:KH-ineq}
\[
  \|F - A(F)\|_{L_2} \le e^{\rm wor-app}_N(A,\bsz)\, \|F\|_\calW \qquad\mbox{for all } F\in\calW.
\]

It is known that kernel methods are optimal in the sense of worst case approximation error among all algorithms using the same set of points. Thus
\[
  e^{\rm wor-app}_N(A^*,\bsz) \le e^{\rm wor-app}_N(A^\dagger,\bsz)
  \qquad\mbox{for all } \bsz\in\bbZ_n^s.
\]
We know how to construct good generating vectors for the truncated trigonometric series $A^\dagger(F)$, either by minimizing an upper bound on the worst case approximation error $e^{\rm wor-app}_N(A^\dagger,\bsz)$ directly (see e.g., \cite{CKNS20,CKNS21,KMN25,KSW06,KWW09c}), or by using \emph{reconstruction lattices} (see e.g., \cite{BKUV17,Kam13,KPV15,PV16,KMNN21}). The generating vectors $\bsz$ obtained in both approaches can then be used in the kernel method $A^*(F)$. An alternative error criterion based on the power function for kernel method is experimented in \cite{KMNSS24}. The approximation error convergence rate is close to $\calO(N^{-\alpha/2})$ in the function space setting (b) above. It has been proved in \cite{BKUV17} that this is the best rate possible using lattice points, compared to the best possible rate close to $\calO(N^{-\alpha})$ that can be achieve by other non-lattice-based algorithms. While this is true, lattice-based algorithms are still competitive because they allow for a simple and highly efficient 
FFT implementation. \emph{Subsampled} lattices have been proved to improve the convergence rate, see \cite{BKPU24}. ``Doubling the rate'' is possible with kernel methods for sufficiently smooth functions, see \cite{KS25}.

\section{Deep Neural Networks as algorithms for function approximation} \label{KKNS:sec:dnn}

We now turn to Deep Neural Networks (DNNs) (see e.g., \cite{GBC16,LBH15}) as nonlinear function approximators. Extending the lattice-based approximation approach from Section~\ref{KKNS:sec:qmc}, we examine DNNs trained on data sampled at lattice points. While some applications provide fixed training data, others permit practitioner-selected points—such as constructing DNN surrogates for costly-to-evaluate functions.
We summarize the regularity results for DNNs from our recent paper \cite{KKNS25a} and provide new extensions.

Instead of working only with real-valued functions $F(\bsy)$ as discussed until now, we will allow vector-valued outputs, i.e., we are interested in problems of the general form
\[ 
\mbox{Given data $\bsy\in Y \subset\bbR^s$, compute observable $G(\bsy)\in Z = \bbR^{N_{\mathrm{obs}}}$}.
\]
For example, if $\bsy$ is a vector of $s$ input parameters to a parametric PDE, then $G(\bsy)$ is a vector of observables obtained from the PDE solution which is a somewhat
smooth function of $\bsy$. The observable may be the average of the PDE
solution over a subset of the domain, in which case $N_{\mathrm{obs}}=1$;
or it may be the vector of values of the PDE solution at the points of a
finite element mesh, in which case $N_{\mathrm{obs}}$ may be the number of
interior finite element mesh points.

\subsection{Non-periodic and Periodic Deep Neural Networks}

For input $\bsy\in Y := [0,1]^s$, we consider two formulations of deep neural networks (DNNs):
the standard feed-forward DNN and a specialized architecture designed for periodic target functions, as recently introduced in \cite{KKNS25a}:
\begin{itemize}
\item [\textnormal{(a)}] Non-periodic DNN:
\begin{align} \label{KKNS:eq:DNN-np}
  G_\theta^{[L]}(\bsy)
  := W_L\, \sigma (\cdots W_2\, \sigma (W_1\, \sigma (W_0\,\bsy + \bsv_0) + \bsv_1) + \bsv_2 \cdots) + \bsv_L.
\end{align}
\item [\textnormal{(b)}] Periodic DNN:
\begin{align} \label{KKNS:eq:DNN-per}
  G_\theta^{[L]}(\bsy)
  := W_L\, \sigma (\cdots W_2\, \sigma (W_1\, \sigma (W_0\,\sin(2\pi\bsy) + \bsv_0) + \bsv_1) + \bsv_2 \cdots) + \bsv_L.
\end{align}
\end{itemize}
In this context, $W_\ell \in \bbR^{d_{\ell+1} \times d_\ell}$ are the weight matrices,
$\bsv_\ell\in\bbR^{d_{\ell+1}\times 1}$ are the bias vectors, and
$L$ is the depth of the neural network. The input layer indexed by $\ell = 0$ inherits its dimension
$d_0 = s$ from the input vector $\bsy$, while the dimension $d_{L+1} = N_{\rm obs}$
of the output layer $\ell = L+1$ is the number of observables. For $1 \leq \ell \leq N$,
a layer is called a hidden layer and has width $d_\ell$.
The entirety $\theta := (W_\ell,\bsv_\ell)_{\ell=0}^L \in \Theta$ of network parameters
describes a  fully-connected neural network
(also known as a multi-layer perceptron). The set of all permissible network parameters is denoted by $\Theta$.
The $N_{\rm DNN} := \sum_{\ell=0}^L (d_{\ell+1}\times d_\ell) + \sum_{\ell=0}^L
 d_{\ell+1}
$
parameters need to be trained. 
The ensemble of depth $L$, the widths $d_\ell$, and other parameters used in the
training phase (see below) are known as the hyperparameters.

The scalar nonlinear activation function $\sigma$ is applied element-wise to vectors obtained from the affine transformations.
Common scalar activation functions are introduced and discussed in Lemma~\ref{KKNS:lem:common} below.
Extending \eqref{KKNS:eq:DNN-np} by an element-wise sine function applied to the input layer yields
\eqref{KKNS:eq:DNN-per}. This sine function may be considered a periodic activation function \cite{AMBLW20}.
In fact, the nonlinear function \eqref{KKNS:eq:DNN-per} is one-periodic with respect to the input
parameters~$\bsy$, as in the theoretical settings (b) and (c) in Section~\ref{KKNS:sec:qmc}.
The particular form \eqref{KKNS:eq:DNN-per}
is designed to handle a target function $G(\bsy)$ that is itself a
function of $\sin(2\pi\bsy)$, such as the quantities of interest from
\cite{HHKKS24,KKKNS22,KKS20,KKS24}.

Training the deep neural network (DNN) involves determining appropriate network parameters
$\theta$, such that $G_\theta^{[L]}(\bsy)$ serves as an accurate approximation of $G(\bsy)$,
particularly at points not included in the training set. The training set comprises points
$\bsy_1, \ldots, \bsy_N \in Y$ and their corresponding function evaluations $G(\bsy_1), \ldots, G(\bsy_N)$,
which act as observables. In this work, we will rely on synthetically generated data (SGD), specifically
employing lattice points for training.

\subsection{Error}

``Training'' is done by minimising an appropriate ``loss function'' $\calJ(\theta)$. Following \cite{KKNS25a} (which in turn follows \cite{LMR20,LMRS21,MR21}), we consider
\begin{align} \label{KKNS:eq:ET}
  \theta^* := \argmin_{\theta\in\Theta} \Big(\calJ(\theta) + \lambda\,\calR(\theta) \Big),
 \quad
 \calJ(\theta) := \frac{1}{N}\sum_{k=1}^N
 \big\| G(\bsy_k) - G_\theta^{[L]}(\bsy_k)\big\|_2^2,
\end{align} 
where we take the vector $2$-norm with respect to the $N_{\rm obs}$ components in the observables, and
where $\calR(\theta)$ is a regularization term which we will tailor according to our theory. The quality of the DNN is expressed in terms of the
\emph{generalization error} (or $L_2$ approximation error)
\begin{align} \label{KKNS:eq:EG}
 \calE_G := 
 \bigg(\int_Y \big\|G(\bsy) - G_\theta^{[L]}(\bsy)\big\|_2^2\,\rd\bsy\bigg)^{1/2}
\le \calE_T + |\calE_G - \calE_T|,
\end{align}
where $\calE_T := \sqrt{\calJ(\theta)}$ is the computable \emph{training error},
and $|\calE_G - \calE_T|$ is the so-called \emph{generalization gap}. 

The square roots in $\calE_G$ and $\calE_T$ make it difficult to analyze the generalization gap directly. So we use the loose estimate $|\calE_G - \calE_T| \le \sqrt{|\calE_G^2 - \calE_T^2|}$, with
\begin{align} \label{KKNS:eq:start}
  |\calE_G^2 - \calE_T^2|
  &= \bigg| \int_Y \big\| G(\bsy) - G_{\theta}^{[L]}(\bsy) \big\|_2^2 \,\rd\bsy
  - \frac{1}{N} \sum_{k=1}^N \big\|G(\bsy_k) - G_{\theta}^{[L]}(\bsy_k)\big\|_2^2\, \bigg| \nonumber\\
  &\le \sum_{p=1}^{N_{\rm obs}} \bigg|
  \int_Y \big( G(\bsy)_p - G_{\theta}^{[L]}(\bsy)_p \big)^2 \,\rd\bsy
  - \frac{1}{N} \sum_{k=1}^N \big(G(\bsy_k)_p - G_{\theta}^{[L]}(\bsy_k)_p\big)^2
  \bigg| \nonumber\\
  &\le \sum_{p=1}^{N_{\rm obs}} e^{\rm wor}_N(\bsz)\,
  \big\| \big(G(\cdot)_p - G_{\theta}^{[L]}(\cdot)_p\big)^2 \big\|_{\calW},
\end{align}
where for each component $p$ we have a generic error bound in the form of \eqref{KKNS:eq:KH-ineq}, namely, the
worst case integration error $e^{\rm wor}_N(\bsz)$ for a lattice rule (with $N$ points and 
generating vector $\bsz$) in a function space $\calW$, multiplied by the norm of the scalar-valued
integrand $F(\bsy) := (G(\bsy)_p - G_{\theta}^{[L]}(\bsy)_p)^2$ in $\calW$.

We can now apply the theoretical settings (a)--(c) from Section~\ref{KKNS:sec:lat-int}, but in each setting we need to show that the integrand $(G(\cdot)_p - G_{\theta}^{[L]}(\cdot)_p)^2$ indeed belongs to the appropriate weighted function space and has a nice bound on its norm. Since the norms in all three settings are defined in terms of mixed derivatives, we need regularity bounds on the DNNs with respect to the input parameters $\bsy\in Y$.

\subsection{Regularity}

Following \cite{KKNS25a}, we assume that there exist positive sequences
$(\beta_j)_{j\ge 1}$, $(R_\ell)_{\ell\ge 1}$, $(A_n)_{n\ge 1}$ for which
the following three key assumptions hold:
\begin{enumerate}
\item The columns of the matrix $W_0$ have bounded vector
    $\infty$-norms
\begin{align} \label{KKNS:eq:beta}
  \|W_{0,:,j}\|_\infty := \max_{1\le p\le d_1} |W_{0,p,j}| \;\le\; \beta_j
  \quad\mbox{for all}\quad j = 1,\ldots,s.
\end{align}
\item The matrices $W_\ell$ for $\ell\ge 1$ have bounded matrix
    $\infty$-norms 
\begin{align} \label{KKNS:eq:R}
  \|W_\ell\|_\infty :=
  \max_{1\le p\le d_{\ell+1}} \sum_{q=1}^{d_\ell} |W_{\ell,p,q}| \;\le\; R_\ell
  \quad\mbox{for all}\quad \ell\ge 1.
\end{align}
\item The activation function $\sigma:\bbR\to\bbR$ is smooth and
    its derivatives satisfy
\begin{align} \label{KKNS:eq:sigma}
  \|\sigma^{(n)}\|_\infty := \sup_{x\in\bbR} |\sigma^{(n)}(x)| \;\le\; A_n
  \quad\mbox{for all}\quad n\ge 1.
\end{align}
\end{enumerate}
The assumption \eqref{KKNS:eq:beta}~and~\eqref{KKNS:eq:R} are analogous to the
assumptions in \cite[Propositions~3.2 and~3.8]{LMRS21}. The assumption
\eqref{KKNS:eq:sigma} allows for a fairly generic smooth activation
functions. 

Specifically, as in \cite{KKNS25a} we will illustrate our results for
activation functions whose derivatives are bounded in the sense of \eqref{KKNS:eq:sigma} by
the common form
\begin{align} \label{KKNS:eq:common}
  A_n = \xi\,\tau^n\,n!
  \qquad\mbox{for some $\xi>0$ and $\tau>0$}.
\end{align}
The following lemma gives the derivative bounds for some popular activation functions.

\begin{lemma} \label{KKNS:lem:common}
For $c>0$, the following generalized activation functions satisfy \eqref{KKNS:eq:sigma} with $A_n$ of the form~\eqref{KKNS:eq:common}:
\begin{align*} 
 {\rm sigmoid}_c(x) &:= \displaystyle\frac{1}{1+e^{-cx}},
 & A_n &= c^n\,n! && \mbox{(i.e., $\xi = 1$, $\tau = c$)};
 \vspace{0.2cm} \nonumber\\
 {\rm tanh}_c(x) &:= \displaystyle\frac{e^{cx}-e^{-cx}}{e^{cx}+e^{-cx}},
 & A_n &= (2c)^n\,n! && \mbox{(i.e., $\xi=1$, $\tau=2c$)};
 \vspace{0.2cm} \\
 {\rm swish}_c(x) &:= \displaystyle\frac{x}{1+e^{-cx}},
 & A_n &= \tfrac{1.1}{c}\,c^n\,n! && \mbox{(i.e., $\xi = \frac{1.1}{c}$, $\tau=c$)}. \nonumber
\end{align*}
\end{lemma}

\begin{proof}
The results for the standard definitions with $c=1$ were proved in \cite{KKNS25a}. We now extend the results to general $c>0$. We have ${\rm sigmoid}_c(x) = {\rm sigmoid}_1(cx)$ and therefore 
\[
 \big|{\rm sigmoid}_c^{(n)}(x)\big| 
  = c^n\,\big|{\rm sigmoid}_1^{(n)}(cx)\big|
  \le c^n\,n!\,.
\]
Similarly, we have $\tanh_c(x) = \tanh_1(cx)$ and therefore 
\[
 \big|\tanh_c^{(n)}(x)\big|
 = c^n\,\big|{\rm tanh}_1^{(n)}(cx)\big|
 \le c^n\times 2^n\,n! = (2c)^n\,n!\,.
\]
The generalized swish is related in a different way to the standard swish, namely,
${\rm swish}_c(x) = {\rm swish}_1(cx)/c$, and in this case
\[
 \big|{\rm swish}_c^{(n)}(x)\big|
 = c^{n-1}\,\big|{\rm swish}_1^{(n)}(cx)\big|
 \le c^{n-1}\times 1.1\,n!
 = \tfrac{1.1}{c}\,c^n\,n!\,.
\]
This completes the proof. \hfill\qed
\end{proof}

Note that ${\rm ReLU}(x) := \max\{x,0\}$ is not smooth and so condition
\eqref{KKNS:eq:sigma} does not hold. However, we have
\[
 {\rm swish}_c(x) \to {\rm ReLU}(x) \quad\mbox{as}\quad c \rightarrow \infty.
\] 
Trivially, ${\rm swish}_c(0) = 0 = {\rm ReLU}(0)$; for $x>0$, ${\rm swish}_c(x) \to x = {\rm ReLU}(x)$ as $c\to\infty$; for $x<0$, ${\rm swish}_c(x) \to 0 = {\rm ReLU}(x)$ as $c\to\infty$.
Both ${\rm ReLU}(x)$ and ${\rm swish_c}(x)$ are unbounded as $x\to+\infty$.

The next lemma shows that the factorial growth in the derivative bounds in Lemma~\ref{KKNS:lem:common} cannot be avoided. This is because the bounds in \cite{KKNS25a} for the case $c=1$ were all derived from the bound of the standard sigmoid function for which we give a lower bound below.

\begin{lemma} \label{KKNS:lem:lower}
For $\sigma(x) = {\rm sigmoid}_1(x) = 1/(1+e^{-x})$ we have for all $n\ge 1$,
\[
  \|\sigma^{(n)}\|_\infty \le \frac{n^n}{(n+1)^{n+1}}\,n! < (n-1)! < n!\,,
\]
and for all odd $n\ge 1$ we have the lower bound
\[
  \|\sigma^{(n)}\|_\infty
  \ge |\sigma^{(n)}(\tfrac{1}{2})| 
  = \frac{2(2^{n+1}-1)\,\zeta(n+1)}{(2\pi)^{n+1}}\,n! 
 > \frac{1}{\pi^n}\,n!\,.
\]
\end{lemma}

\begin{proof}
The upper bound was proved in \cite{KKNS25a} using the identity \cite[formula~(15)]{MinWil93}
\begin{align*} 
  \sigma^{(n)}(x) = \sum_{k=1}^n (-1)^{k-1} E(n,k-1)\,[\sigma(x)]^k\,[1-\sigma(x)]^{n+1-k},
  \quad n\ge 1,
\end{align*}
where $E(n,k)$ is the Eulerian number. To obtain a lower bound we consider $x = 1/2$,
\begin{align*}
  \|\sigma^{(n)}\|_\infty \ge |\sigma^{(n)}(\tfrac{1}{2})|
  = \frac{1}{2^{n+1}} \bigg|\sum_{k=1}^n (-1)^{k-1} E(n,k-1)\bigg|
  = \frac{1}{2^{n+1}} \frac{2^{n+1}(2^{n+1}-1)\,|B_{n+1}|}{n+1},
\end{align*}
where we used the property that an alternating sum of Eulerian numbers is related a Bernoulli number
by \cite[Formula~(26.14.11)]{NIST}. Note that $B_{n+1}=0$ for even $n$. So we consider odd $n = 2m-1$ for $m\ge 1$ and use \cite[Formula~(25.6.2)]{NIST} for $|B_{n+1}| = |B_{2m}|$ to get
\begin{align*}
 |\sigma^{(n)}(\tfrac{1}{2})|
 = \frac{2^{n+1}-1}{n+1}\,\frac{2\zeta(2m)\,(2m)!}{(2\pi)^{2m}}
 = \frac{2(2^{n+1}-1)\,\zeta(n+1)}{(2\pi)^{n+1}}\,n! 
 > \frac{1}{3\pi^{n-1}}\,n!
 > \frac{1}{\pi^n}\,n!\,,
\end{align*}
where the Riemann zeta function satisfies $\zeta(n+1)\ge \zeta(2) = \pi^2/6$.
\hfill\qed
\end{proof}

In order to state our regularity results for the DNNs, we need to introduce further notation.
Let $\bbN_0 := \{0,1,2,\ldots\}$ be the set of nonnegative integers. For a
multiindex $\bsnu = (\nu_j)_{j\ge 1} \in \bbN_0^\infty$ we define its
order to be $|\bsnu| := \sum_{j\ge 1} \nu_j$, and we consider the index
set $\calI := \{\bsnu \in \bbN_0^\infty : |\bsnu| < \infty\}$. We write
$\bsm\le\bsnu$ when $m_j\le \nu_j$ for all $j\ge 1$.
We define
the mixed partial derivative operator with respect to the parametric
variables $\bsy$ by $\partial^\bsnu := \prod_{j\ge 1}
(\frac{\partial}{\partial y_j})^{\nu_j}$ and $\partial^{\bsnu_\setu} :=
\prod_{j\in\setu} (\frac{\partial}{\partial y_j})^{\nu_j}$.
We will need the \emph{Stirling number of the second kind} which are defined
for integers $n$ and $k$ by $\calS(n,0):=\delta_{n,0}$ (i.e., it is 1 if $n=0$
and is $0$ if $n\ge 1$), $\calS(n,k) := 0$ for $k> n$, and otherwise
\[
  \calS(n,k) := \frac{1}{k!}\sum_{i=0}^k (-1)^{k-i} \binom{k}{i} i^n, \qquad n\ge k.
\]

\begin{theorem}[{{\cite[Theorem~2.2]{KKNS25a}}}] \label{KKNS:thm:reg}
Let the sequences $(\beta_j)_{j\ge 1}$, $(R_\ell)_{\ell\ge 1}$,
$(A_n)_{n\ge 1}$ be defined as in \eqref{KKNS:eq:beta}, \eqref{KKNS:eq:R},
\eqref{KKNS:eq:sigma}, respectively, with $A_n$ given by \eqref{KKNS:eq:common} in terms of parameters $\xi,\tau>0$. For
depth $L\ge 1$, any component $1\le p\le N_{\rm obs}$, and any
multiindex $\bsnu\in\calI$ $($including $\bsnu=\bszero$$)$, we have the
following regularity bounds:
\begin{itemize}
\item [\textnormal{(a)}] The non-periodic DNN defined in
    \eqref{KKNS:eq:DNN-np} satisfies
\begin{align} \label{KKNS:eq:reg-np}
  |\partial^\bsnu G_\theta^{[L]}(\bsy)_p|
  \le C_L\, \,|\bsnu|!\, 
  \prod_{j=1}^s (S_L\,\beta_j)^{\nu_j}.
  \qquad\qquad\qquad\qquad\quad
\end{align}

\item [\textnormal{(b)}] The periodic DNN defined in
    \eqref{KKNS:eq:DNN-per} satisfies
\begin{align} \label{KKNS:eq:reg-per}
  |\partial^\bsnu G_\theta^{[L]}(\bsy)_p|
  \le C_L\, (2\pi)^{|\bsnu|}
  \sum_{\bsm\le\bsnu} |\bsm|!\,
  \prod_{j=1}^s \Big((S_L\,\beta_j)^{m_j}\,\calS(\nu_j,m_j) \Big).
\end{align}
\end{itemize}
In both cases we define 
\begin{align} \label{KKNS:eq:CSP}
 C_L := \max\bigg\{
 \|G^{[L]}_\theta\|_{\infty},\;
 \frac{P_L}{S_L} \bigg\},
  \quad
  S_L := \tau \sum_{\ell=0}^{L-1} P_\ell, \quad
  P_\ell := \prod_{t=1}^\ell (\xi\,\tau\, R_t).
\end{align}
\end{theorem}

We can restrict the network parameters of the DNNs to match the regularity
features of the target functions and then obtain a theoretical bound on the norms in \eqref{KKNS:eq:start}.
Regularity bounds of the form \eqref{KKNS:eq:tar-np} for target functions have appeared in PDE applications in e.g., \cite{GGKSS19,GKNSSS15,GKNSS18b,GKKSS24,KN16,KSS12},
while bounds of the form \eqref{KKNS:eq:tar-per} have appeared in e.g., \cite{DKLNS14,HHKKS24,KKKNS22,KKS20,KKS24}. In these PDE applications, a sequence $b_j$ captures the decaying magnitude of the random field fluctuations corresponding to the $j$th basis function as $j$ increases. Correspondingly, this sequence $b_j$ reflects the decaying importance of the variables $y_j$ in our weighted function spaces.

We emphasize that the bounds in Theorem~\ref{KKNS:thm:reg} for the DNNs are completely general. If the target function is not from the PDE applications mentioned above, then one should first obtain a regularity bound for the target function and then devise a possibly different strategy to restrict the network parameters to match the regularity features of the target function.

\begin{theorem}[{{\cite[Theorem~3.2]{KKNS25a}}}] \label{KKNS:thm:diff}
Let the sequences $(\beta_j)_{j\ge 1}$, $(R_\ell)_{\ell\ge 1}$,
$(A_n)_{n\ge 1}$ be defined as in \eqref{KKNS:eq:beta}, \eqref{KKNS:eq:R},
\eqref{KKNS:eq:sigma}, respectively, with $A_n$ given by \eqref{KKNS:eq:common} in terms of parameters $\xi,\tau>0$. 
For depth $L\ge 1$, let $C_L$ and $S_L$ be defined as in \eqref{KKNS:eq:CSP}.
Given a sequence $(b_j)_{j\ge 1}$ and a constant $C>0$, we 
\begin{itemize}
\item restrict the elements of the matrices $W_1,\ldots,W_{L-1}$ so
    that for some constant $\rho>0$,
\begin{align} \label{KKNS:eq:demand1}
 & R_\ell \le \rho
  \quad\mbox{for all $\ell= 1, \ldots,L-1$},
 \\
 \label{KKNS:eq:demand1b}
 \Big( \mbox{thus}\;\; S_L \le \tau\,L \;\;\mbox{if}
 &\;\; \xi\,\tau\,\rho = 1, \;\;\mbox{and}\;\;
 S_L \le \tau\frac{(\xi\,\tau\,\rho)^L-1}{\xi\,\tau\,\rho-1}
 \;\;\mbox{otherwise} \Big),
\end{align}
\item restrict the elements of the matrix $W_0$ so that
\begin{align} \label{KKNS:eq:demand2}
  \beta_j \le \frac{b_j}{S_L}
  \quad\mbox{for all $j=1,\ldots,s$},
\end{align}
\item restrict the matrices $W_1,\ldots, W_L$ and vector $\bsv_L$ so that
\begin{align} \label{KKNS:eq:demand3}
 C_L \le C.
\end{align}
\end{itemize}
\begin{itemize}
\item [\textnormal{(a)}] Suppose a non-periodic target function
    $G(\bsy)$ satisfies the regularity bound: for all multiindices
    $\bsnu\in\calI$ and all components $1\le p\le N_{\rm obs}$,
\begin{align} \label{KKNS:eq:tar-np}
    |\partial^\bsnu G(\bsy)_p| \le C\,|\bsnu|!\, \prod_{j=1}^s b_j^{\nu_j}.
\end{align}
Then the non-periodic DNN \eqref{KKNS:eq:DNN-np} with restrictions
\eqref{KKNS:eq:demand1}--\eqref{KKNS:eq:demand3} satisfies the same regularity
bound \eqref{KKNS:eq:tar-np}, and for the norm \eqref{KKNS:eq:norm-np} we have for all components $1\le p\le N_{\rm obs}$,
\begin{align}
 \big\|\big(G(\cdot)_p- G^{[L]}_\theta(\cdot)_p\big)^2\big\|_{\calW_{1,\bsgamma}}^2
 \le 16\,C^4
 \sum_{\setu\subseteq\{1:s\}} \frac{1}{\gamma_\setu} \bigg(
 (|\setu|+1)!\, \prod_{j\in\setu} b_j
 \bigg)^2. \label{KKNS:eq:final-np}
\end{align}

\item [\textnormal{(b)}] Suppose a periodic target function $G(\bsy)$
    satisfies the regularity bound: for all mutiindices
    $\bsnu\in\calI$ and all components $1\le p\le N_{\rm obs}$,
\begin{align} \label{KKNS:eq:tar-per}
  |\partial^\bsnu G(\bsy)_p|
  \le C\,(2\pi)^{|\bsnu|} \sum_{\bsm\le\bsnu} |\bsm|!\,
  \prod_{j=1}^s \Big(b_j^{m_j}\,\calS(\nu_j,m_j)\Big).
\end{align}
Then the periodic DNN \eqref{KKNS:eq:DNN-per} with restrictions
\eqref{KKNS:eq:demand1}--\eqref{KKNS:eq:demand3} satisfies the same regularity
bound \eqref{KKNS:eq:tar-per}, and for the norm \eqref{KKNS:eq:norm-per} we have for all components $1\le p\le N_{\rm obs}$,
\begin{align}
 &\big\|\big(G(\cdot)_p - G^{[L]}_\theta(\cdot)_p\big)^2\big\|_{\calW_{\alpha,\bsgamma}}^2 \nonumber\\
 &\le 16\,C^4\!\!
 \sum_{\setu\subseteq\{1:s\}} \!\!\!\!
 \frac{(2\pi)^{2\alpha|\setu|}}{\gamma_\setu} \bigg(
 \sum_{\bsm_\setu\le\bsalpha_\setu} (|\bsm_\setu|+1)!\,
 \prod_{j\in\setu} \Big(b_j^{m_j}\,\calS(\alpha,m_j)\Big)
 \bigg)^2, \label{KKNS:eq:final-per}
\end{align}
where $\bsalpha_\setu$ denotes a multiindex with $|\setu|$ components all equal to $\alpha$.

\item [\textnormal{(c)}] Continuing from \textnormal{(b)}, for the norm
    \eqref{KKNS:eq:norm-K} we have for all components
$1\le p\le N_{\rm obs}$,
\begin{align}
 &\big\|\big(G(\cdot)_p - G^{[L]}_\theta(\cdot)_p\big)^2\big\|_{\calK_{\alpha,\bsgamma}} \nonumber \\
 &\le 4\,C^2
 \max_{\setu\subseteq\{1:s\}} \!\!
 \frac{(2\pi)^{\alpha|\setu|}}{\gamma_\setu} \!\!
 \sum_{\bsm_\setu\le\bsalpha_\setu} (|\bsm_\setu|+1)!\,
 \prod_{j\in\setu} \Big(b_j^{m_j}\,\calS(\alpha,m_j) \Big).
 \label{KKNS:eq:final-K}
\end{align}
\end{itemize}
\end{theorem}

The sequence $b_j$  appears in the norm bounds \eqref{KKNS:eq:final-np}, \eqref{KKNS:eq:final-per}, \eqref{KKNS:eq:final-K} and directly informs our choice of weights $\gamma_\setu$  in the theorem below. Crucially, while larger weights $\gamma_\setu$  reduce these norm bounds, they simultaneously increase the worst case error bounds \eqref{KKNS:eq:wce-np}, \eqref{KKNS:eq:wce-per}, \eqref{KKNS:eq:wce-K}. Since the generalization error is bounded by the product of these factors (see \eqref{KKNS:eq:start}), we optimize $\gamma_\setu$  to balance this trade-off. This framework enables the construction of lattice generating vectors achieving  convergence rates independent of the input dimension~$s$ across all three settings.

While the input dimension $s$ is fixed throughout, we consider an infinite sequence $b_j$ which is summable in order to have the constant in our error bound independent of~$s$. The ``summability exponent'' defined below then determines the convergence rates.
Weights of the form \eqref{KKNS:eq:weight-np} are referred to as POD (``product and order dependent'') weights, while weights of the form \eqref{KKNS:eq:weight-per} and \eqref{KKNS:eq:weight-K} are referred to as SPOD (``smoothness-driven product and order dependent'') weights.

\begin{theorem}[{{\cite[Theorem~3.3]{KKNS25a}}}] \label{KKNS:thm:err}
Under the same conditions as in Theorem~\ref{KKNS:thm:diff} and for the respective settings \textnormal{(a), (b), (c)}, assume further
that there is a ``summability exponent'' $p^* \in (0,1)$ such that
$\sum_{j\ge 1} b_j^{p^*} < \infty$.
\begin{itemize}
\item [\textnormal{(a)}] Construct as in Section~\ref{KKNS:sec:qmc} setting \textnormal{(a)} a generating vector $\bsz$ for a randomly-shifted lattice rule in weighted Sobolev space with weights
\begin{align} \label{KKNS:eq:weight-np}
 \gamma_\setu := \bigg(
 (|\setu|+1)!
 \prod_{j\in\setu} \!\bigg(\!
 \sqrt{\frac{(2\pi)^{2\lambda}}{2\zeta(2\lambda)\,2^\lambda}}\, b_j\bigg)
 \bigg)^{\frac{2}{1+\lambda}},
 \,
 \lambda :=
 \begin{cases}
 \frac{1}{2-2\delta}, \delta\in (0,\frac{1}{2}), \!\!\!\!\!
 &\mbox{if } p^* \in (0, \frac{2}{3}], \\
 \frac{p^*}{2-p^*} 
 &\mbox{if } p^* \in (\frac{2}{3}, 1).
 \end{cases}
\end{align}
If a non-periodic DNN \eqref{KKNS:eq:DNN-np} is trained using these lattice
points, and the training error reaches the threshold ${\tt tol} \in
(0,1)$, then the generalization error satisfies
\begin{align} \label{KKNS:eq:rate-np}
  \calE_G 
  \le {\tt tol} + \calO\big(N^{-r/2}\big),
  \qquad
  r := \min\Big(1-\delta,\;\frac{1}{p^*}-\frac{1}{2}\Big),
  \quad \delta\in (0,\tfrac{1}{2}).
\end{align}

\item [\textnormal{(b)}] Construct as in Section~\ref{KKNS:sec:qmc} setting \textnormal{(b)} a generating vector $\bsz$ for a lattice rule in weighted Hilbert Korobov space with weights
\begin{align} \label{KKNS:eq:weight-per}
 \gamma_\setu :=
 (2\pi)^{2\alpha|\setu|} \!\!\!\!\!
 \sum_{\bsm_\setu\le\bsalpha_\setu} \!\!\!
 \bigg(
 (|\bsm_\setu|+1)!
 \prod_{j\in\setu} \!\!
 \frac{b_j^{m_j}\calS(\alpha,m_j)}{\sqrt{2\zeta(2\alpha\lambda)}}
 \bigg)^{\frac{2}{1+\lambda}},
 \,
 \alpha := \left\lfloor \tfrac{1}{p^*} + \tfrac{1}{2} \right\rfloor,
 \,
 \lambda := \tfrac{p^*}{2-p^*}.
\end{align}
If a periodic DNN \eqref{KKNS:eq:DNN-per} is trained using these lattice
points, and the training error reaches the threshold ${\tt tol} \in
(0,1)$, then the generalization error satisfies
\begin{align} \label{KKNS:eq:rate-per}
  \calE_G 
  \le {\tt tol} + \calO\big(N^{-r/2}\big),
  \qquad
  r := \frac{1}{p^*}-\frac{1}{2}.
\end{align}

\item [\textnormal{(c)}] Construct as in Section~\ref{KKNS:sec:qmc} setting \textnormal{(c)} a generating vector $\bsz$ for a lattice rule in weighted non-Hilbert Korobov space with weights
\begin{align} \label{KKNS:eq:weight-K}
 \gamma_\setu :=
 (2\pi)^{\alpha|\setu|}
 \sum_{\bsm_\setu\le\bsalpha_\setu}
 (|\bsm_\setu|+1)!\,
 \prod_{j\in\setu} \big( b_j^{m_j}\calS(\alpha,m_j)\big), \quad
 \alpha := \left\lfloor \tfrac{1}{p^*} \right\rfloor + 1.
\end{align}
If a periodic DNN \eqref{KKNS:eq:DNN-per} is trained using these lattice
points, and the training error reaches the threshold ${\tt tol} \in
(0,1)$, then the generalization error satisfies
\begin{align} \label{KKNS:eq:rate-K}
  \calE_G 
  \le {\tt tol} + \calO\big(N^{-r/2}\big),
  \qquad
  r := \frac{1}{p^*}.
\end{align}
\end{itemize}
In all three cases the implied constant in the big-$\calO$ bound is
independent of $s$.

If the condition \eqref{KKNS:eq:demand2} is not achieved after
training, i.e., if $\beta_j \le b_j/S_L$ is not true for some $j$, with $S_L$ defined in \eqref{KKNS:eq:CSP},
but instead there is a constant $\kappa\ge 1/s_L$ such that
\begin{align} \label{KKNS:eq:kappa}
  \beta_j \le \kappa\,b_j
  \quad\mbox{for all } j = 1,\ldots,s,
\end{align}
then the bounds \eqref{KKNS:eq:final-np}, \eqref{KKNS:eq:final-per},
\eqref{KKNS:eq:final-K} hold with each $b_j$ replaced by $\kappa\,S_L\,b_j$.
Even though the lattice training points were constructed with the weights
\eqref{KKNS:eq:weight-np}, \eqref{KKNS:eq:weight-per}, \eqref{KKNS:eq:weight-K} based on
the original $b_j$, the big-$\calO$ bounds \eqref{KKNS:eq:rate-np},
\eqref{KKNS:eq:rate-per}, \eqref{KKNS:eq:rate-K} hold now but with enlarged implied
constants that depend on $\kappa\,S_L$. In particular, the enlarged constants for \eqref{KKNS:eq:rate-np} and \eqref{KKNS:eq:rate-per} are still bounded independently of $s$, but the enlarged constant for \eqref{KKNS:eq:rate-K} grows exponentially with~$s$.
\end{theorem}

The results are obtained by combining \eqref{KKNS:eq:final-np} with \eqref{KKNS:eq:wce-np}, \eqref{KKNS:eq:final-per}
with \eqref{KKNS:eq:wce-per}, and \eqref{KKNS:eq:final-K} with \eqref{KKNS:eq:wce-K}, and then selecting the weights
and other parameters, such as $\alpha$ and $\lambda$, in relation to $p^*$ to ensure that the resulting
constant remains finite in each case. The derivation of weights and convergence rates follows similar
reasoning to that presented in several recent works, including \cite{KN16,KSS12} for (a), \cite{KKKNS22}
for (b), and \cite{KKS20} for (c). Since the detailed derivations were omitted in \cite{KKNS25a}, we
provide the complete proof in the appendix for clarity and completeness.

Note the theoretical advantage of setting (c): the value of $r$ is larger, and hence yields faster convergence.

\subsection{Tailored regularization}

We train our DNN to obtain
\begin{align*}
  \theta^* := \argmin_{\theta\in\Theta}
  \bigg(
  \frac{1}{N} \sum_{k=1}^N \big\|
  G(\bsy_k) - G_{\theta}^{[L]}(\bsy_k) \big\|_2^2
  + \lambda\, \|\theta\|_2^2
  + \lambda_1\,\calR_1(\theta)
  \bigg).
\end{align*}
When $\lambda_1 = 0$ we have the
standard $\ell_2$ regularization. 
When $\lambda_1 >0$ we have the tailored
regularization from \cite{KKNS25a} which is designed on the basis of Theorem~\ref{KKNS:thm:diff}:
\begin{align}
  \calR_1(\theta) &:=
  \frac{1}{s} \sum_{j=1}^s \frac{1}{d_1} \sum_{p=1}^{d_1} \Big(W_{0,p,j}^2\,
  \frac{L^2}{b_j^2}\Big)^{m/2},
  \label{KKNS:eq:R1-def} 
\end{align}
with a positive even integer $m = 6$. In particular, $\calR_1(\theta)$ is chosen to ``encourage'' the condition~\eqref{KKNS:eq:demand2}, i.e., $\|W_{0,j}\|_\infty \le \beta_j \le b_j/S_L$ for all $j=1,\ldots,s$,
which is equivalent to
\begin{align} \label{KKNS:eq:from}
  \max_{1\le j\le s} \max_{1\le p\le d_1} |W_{0,p,j}|\, \frac{S_L}{b_j} \le 1.
\end{align}
To arrive at our design of the regularization term \eqref{KKNS:eq:R1-def} from \eqref{KKNS:eq:from}, we replace $S_L$ by $L$ to at least capture the fact that $S_L$ generally increases with increasing~$L$, and we estimate $\max_{1\le i\le n} |a_i| = \lim_{m\to\infty} (\frac{1}{n}\sum_{i=1}^n |a_i|^m)^{1/m}$ with just $m=6$, where the even power allows us to omit the absolute values.
The smoothness of $\calR_1(\theta)$ permits gradient computation $\nabla_\theta \calR_1(\theta)$ via automatic differentiation.

Numerical experiments in \cite{KKNS25a} showed that, for a periodic algebraic function and the sigmoid activation function, this tailored regularization is successful in encouraging the sequence $\beta_j$ to decay like~$b_j$, and as the theory predicts, this improved the generalization gap compared to just the standard $\ell_2$ regularization.

\section{Numerical experiments} \label{KKNS:sec:num}

In this study, we extend the experiments from \cite{KKNS25a} to explore different activation functions beyond the sigmoid function. Specifically, we examine the ${\rm swish}_c$ function with various parameters $c\in \{1, 5, 25\}$
and the ReLU function, which is not smooth and therefore falls outside the scope of the existing theory.
Notably, ${\rm swish}_c(x)$ converges to ${\rm ReLU}(x)$ as $c\to\infty$. 
Unlike the bounded sigmoid function, both ${\rm swish}_c(x)$ and ${\rm ReLU}(x)$ are unbounded as $x\to\infty$. 
Figure~\ref{KKNS:fig:swish} provides a visual comparison of these functions.
From Lemma~\ref{KKNS:lem:common} and \eqref{KKNS:eq:demand1}--\eqref{KKNS:eq:demand1b}, we see that if $R_\ell\le \rho$ for all $\ell=1,\ldots,L-1$ then with ${\rm swish}_c(x)$ we have 
\[
  S_L \le c \frac{(1.1\rho)^L-1}{1.1\rho-1}.
\]
Thus the factor $S_L$ appearing in \eqref{KKNS:eq:reg-np}--\eqref{KKNS:eq:reg-per} grows linearly with increasing $c$, which is then compounded by the product. We are interested to see whether this affects the performance in practice.

\begin{figure}[t]
\begin{center}
 \includegraphics[width=0.8\textwidth]{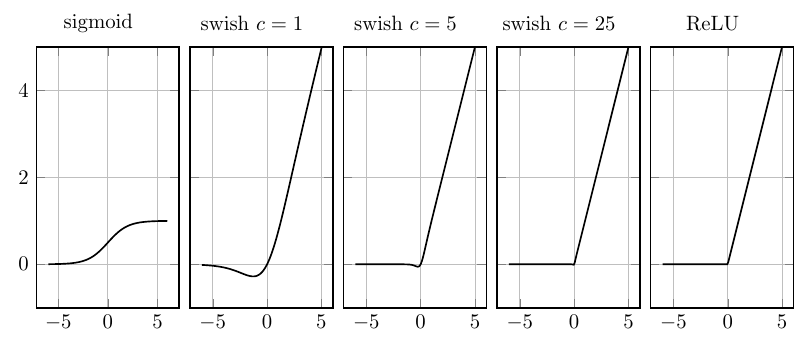}
 \caption{Graphs of different activation functions.}
 \label{KKNS:fig:swish}
\end{center}
\end{figure}

As in \cite{KKNS25a}, we investigate the periodic algebraic function
\begin{align*} 
  G(\bsy) := \frac{1}{a(\bsy)}, \quad
  a(\bsy) := 1 + \sum_{j=1}^s \sin(2\pi y_j)\,\psi_j,
  \quad \psi_j := \frac{\eta}{j^q} \mbox{ for } q>1,
\end{align*}
which satisfies the regularity bound
\eqref{KKNS:eq:tar-per} with $b_j = \psi_j/a_{\min}$ and $C = 1/a_{\min}$, 
where $a_{\min} := 1 - \eta\,\zeta(q)$. For the experiments, we select
$\eta = 0.5$ and $q = 2.5$. This corresponds to $p^* \approx 1/q = 1/2.5$ in Theorem~\ref{KKNS:thm:err}, giving the predicted convergence rate $r/2 \approx 0.5$, $1$, $1.25$ in
settings (a), (b), (c), respectively. 
We train a periodic DNN \eqref{KKNS:eq:DNN-per} with the five activation functions in Figure~\ref{KKNS:fig:swish} using the same two sets of hyperparameters as in \cite{KKNS25a}:
\begin{enumerate}
\item 
  $L = 3$,\;\; $N_{\rm obs} = 1$, $d_\ell = 32$, $s = 50$\; ($3777$ parameters);
\item 
  $L = 12$, $N_{\rm obs} = 1$, $d_\ell = 30$, $s = 50$\;  ($11791$ parameters).
\end{enumerate}
Although these networks may be considered small, recent studies have demonstrated the significant success of ``tiny" neural networks in computer graphics, see \cite{MESK22,MRNK21,VSWAEL23}.

The DNN is trained by running a full-batch Adam (adaptive moment estimation) optimization algorithm \cite{Adam}, which is a gradient descent method.
Our algorithm is implemented in the Python-based open-source deep learning framework PyTorch \cite{PyTorch}.
A random Glorot initialization determines
the initial parameter configuration and the training runs with
a learning rate (step size) of $10^{-4}$. The optimization process terminates either
when the training error $\calE_T = \sqrt{\calJ(\theta)}$ (see \eqref{KKNS:eq:ET}) reaches the
threshold ${\tt tol} = 10^{-3}$, or after completing a maximum number of epochs (steps)
of $40000$.

We employ a randomly shifted version of an ``off-the-shelf'' embedded lattice point set, constructed according to \cite{CKN06}, as our training points
$\{\bsy_k\}_{k\ge 1}$. The network is progressively trained using $N = 2^5, \ldots, 2^{12}$ points. We then compare the training error $\calE_T$ to an
estimate of the generalization error $\calE_G$ (see \eqref{KKNS:eq:EG}) provided
by
\[
  \calE_G
  \approx
  \widetilde{\calE}_G
  := \bigg(\frac{1}{M} \sum_{i=1}^M \big(G(\bst_i) - G_{\theta}^{[L]}(\bst_i) \big)^2 \bigg)^{1/2},
\]
where a different and independent random shift is used with the same embedded lattice rule, but with a much higher number of points $M = 2^{15} \gg N$. Finally, we compute an estimate of the generalization gap by
$|\widetilde{\calE}_G - \calE_T|$. We know from Theorem~\ref{KKNS:thm:err} that
\begin{align*}
  \calE_G
  \le \calE_T + |\calE_G - \calE_T|
  \le 10^{-3} + \calO\big(N^{-r/2}\big).
\end{align*}
When tailored regularization is applied, we aim to observe a consistently lower generalization error $\widetilde{\calE}_G$ and a faster convergence rate for the gap $|\widetilde{\calE}_G - \calE_T|$.

\begin{figure}[t]
\begin{center}
 \includegraphics[width=\textwidth]{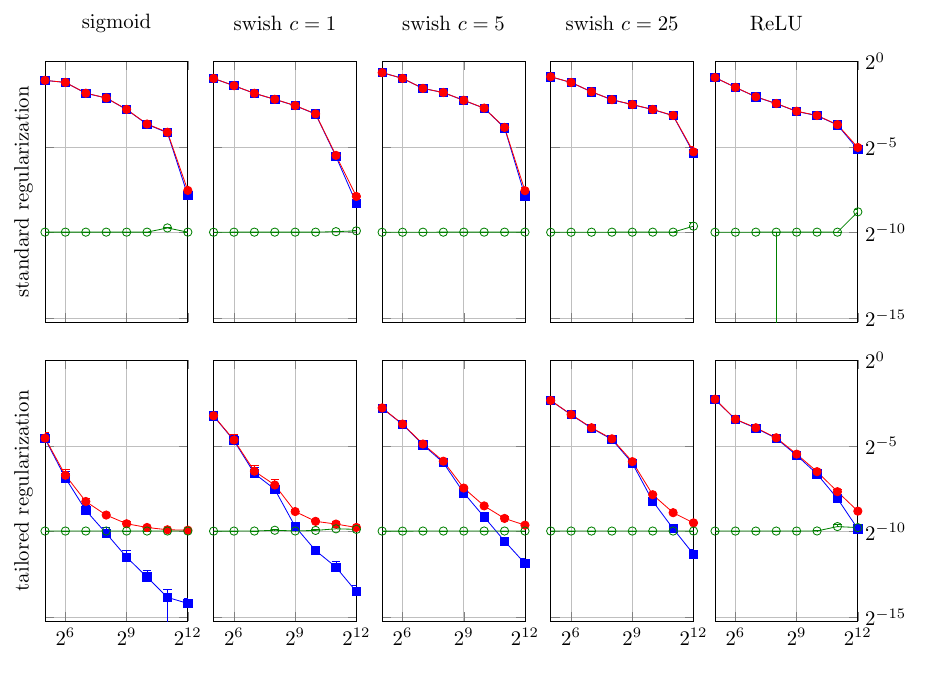}
 \caption{Values of $\calE_T$ (\textcolor{green!50!black}{$\circ$} green circles), $\widetilde{\calE}_G$ (\textcolor{red}{$\bullet$} red dots), and $|\widetilde{\calE}_G - \calE_T|$ (\textcolor{blue}{\scalebox{0.7}{$\blacksquare$}} blue squares) as $N$ increases, with 
 $L = 3$, $N_{\rm obs} = 1$, $d_\ell = 32$, $s = 50$.}
 \label{KKNS:fig:L=3}
\end{center}
\end{figure}

\begin{figure}[t]
\begin{center}
 \includegraphics[width=\textwidth]{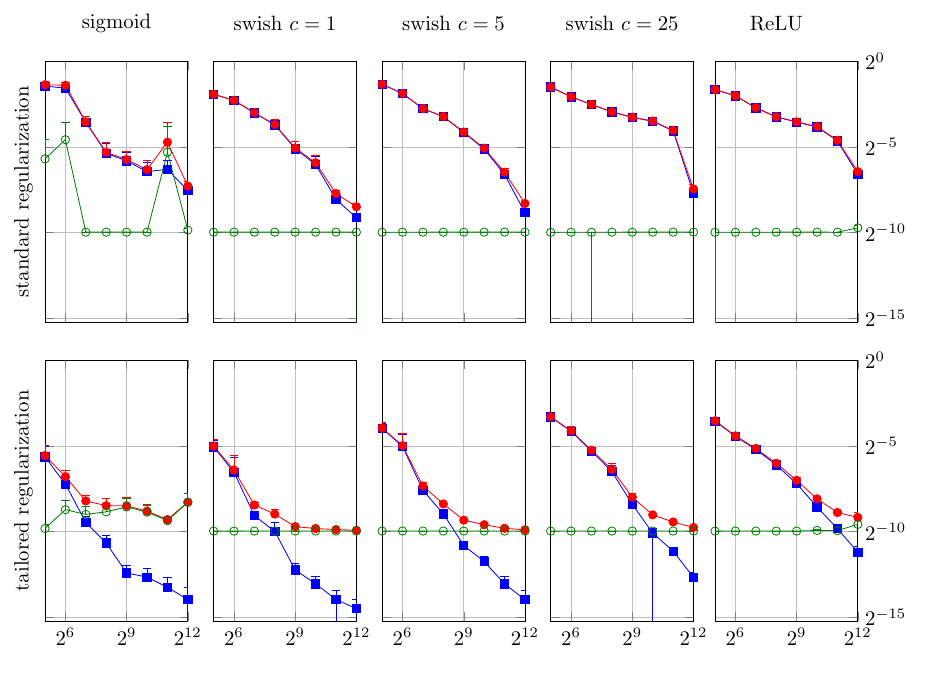}
 \caption{Values of $\calE_T$ (\textcolor{green!50!black}{$\circ$} green circles), $\widetilde{\calE}_G$ (\textcolor{red}{$\bullet$} red dots), and $|\widetilde{\calE}_G - \calE_T|$ (\textcolor{blue}{\scalebox{0.7}{$\blacksquare$}} blue squares) as $N$ increases, with 
 $L = 12$, $N_{\rm obs} = 1$, $d_\ell = 30$, $s = 50$.}
 \label{KKNS:fig:L=12}
\end{center}
\end{figure}

In Figure~\ref{KKNS:fig:L=3}, for the first set of hyperparameters ($L=3$), we plot 
the values of the training error $\calE_T$~(\textcolor{green!50!black}{$\circ$} green circles), the estimated generalization error $\widetilde{\calE}_G$~(\textcolor{red}{$\bullet$} red dots), and the generalization gap $|\widetilde{\calE}_G - \calE_T|$~(\textcolor{blue}{\scalebox{0.7}{$\blacksquare$}} blue squares) for $N = 2^5,\ldots,2^{12}$. Each of these quantities is averaged over $20$
different random Glorot initializations of the network and random shifts of the lattice rule. 
The top five graphs (labelled ``standard regularization'', corresponding to our five activation functions) were obtained with regularization parameter $\lambda = 10^{-8}$ but without our tailored regularization (so $\lambda_1= 0$). The bottom five graphs (labelled ``tailored regularization'') were obtained with $\lambda = 10^{-8}$ and our tailored regularization $\lambda_1 = 10^{-8}$. The values of $\lambda$ and $\lambda_1$ were chosen by trial and error to ensure consistently strong performance.
In Figure~\ref{KKNS:fig:L=12} we present the same plots for the second set of hyperparameters ($L=12$).

Across all five activation functions in Figures~\ref{KKNS:fig:L=3} and~\ref{KKNS:fig:L=12}, we observe that our tailored regularization consistently improves performance over standard $\ell_2$ regularization. With standard regularization, none of the estimated generalization errors $\widetilde{\calE}_G$~(\textcolor{red}{$\bullet$}) reaches our desired threshold ${\tt tol} = 10^{-3}$ at $N = 2^{12}$. 
In contrast, with tailored regularization, it is evident that the performance of swish with~$c \in \{1, 5, 25\}$ deteriorates as $c$ increases, approaching the performance of ReLU which is the worst among the five. This observation aligns  with our theoretical error bound, where the factor $S_L$ increases with~$c$, and ReLU's lack of smoothness means the theory does not apply.

In Figure~\ref{KKNS:fig:L=3} ($L=3$), the sigmoid activation function emerges as the winner. This case has been already presented in \cite[Figure~1]{KKNS25a}, where we observed that the estimated generalization gap $|\widetilde{\calE}_G - \calE_T|$~(\textcolor{blue}{\scalebox{0.7}{$\blacksquare$}}) decays at a rate between $\calO(N^{-1})$ and $\calO(N^{-2})$, and quickly drops below ${\tt tol} =10^{-3}$ by $N=2^9$, with $\widetilde{\calE}_G$~(\textcolor{red}{$\bullet$}) fairly close to ${\tt tol} = 10^{-3}$ from $N = 2^{10}$ onwards.

The winner in Figure~\ref{KKNS:fig:L=12} ($L=12$) is the standard swish activation function ($c=1$), where the gap $|\widetilde{\calE}_G - \calE_T|$~(\textcolor{blue}{\scalebox{0.7}{$\blacksquare$}}) decays at a rate between $\calO(N^{-1})$ and $\calO(N^{-2})$, and falls significantly below ${\tt tol} =10^{-3}$ already at $N=2^9$, with $\widetilde{\calE}_G$~(\textcolor{red}{$\bullet$}) fairly close to ${\tt tol} = 10^{-3}$ from $N = 2^9$. In fact, all variants of swish and ReLU perform better than sigmoid. As noted in \cite[Figure~2]{KKNS25a} for the sigmoid case the training error $\calE_T$~(\textcolor{green!50!black}{$\circ$}) cannot reach ${\tt tol} = 10^{-3}$. A strategy of gradually reducing the learning rate during training might enhance the performance of the sigmoid function in this scenario. However, it is important to note that the swish function consistently performs exceptionally well with a constant learning rate.

In this survey, we do not present results for standard Monte Carlo (MC) points. As demonstrated in \cite[Figures~1 and~2]{KKNS25a}, MC performs surprisingly similar to QMC. One possible explanation is that tailored regularization reduces the variance of $(G - G_\theta^{[L]})^2$. Also, as demonstrated in \cite[Figures~3 and~4]{KKNS25a}, the sequence $(\beta_j)$ with tailored regularization decays closely following the sequence $(b_j/L)$ for all activation functions; we do not repeat those plots here.

\section{Concluding remarks} \label{KKNS:sec:conc}

In this survey and in \cite{KKNS25a}, the focus is on scenarios where target
functions are smooth but involve numerous variables and are expensive to
evaluate. As outlined in \cite{ABDM24}, much of the
recent research on DNN approximation theory centers
on proving \emph{existence theorems} through \emph{emulation}. However, there
remains a notable gap between theoretical results and practical performance
(e.g., \cite{AD21,GV24}), motivating efforts to bridge this gap via
\emph{practical existence theory} (e.g., \cite{ABDM25}).
Our Theorem~\ref{KKNS:thm:err} presupposes the existence of a sequence $(b_j)_{j \geq 1}$
that encapsulates the regularity of the target function, a concept analogous to the
holomorphic function framework explored in \cite{ABDM24} and further discussed
in \cite{LMRS21,SZ19}. Unlike \cite{ABDM24}, which addresses the scenario of
\emph{unknown anisotropy} where learning algorithms operate without explicit
reliance on $(b_j)_{j \geq 1}$, our approach explicitly integrates knowledge of
this sequence into the design of lattice training points and a practical regularization term.
We show theoretically that, by restricting the network parameters to match the regularity
features of the target functions, good generalization error bounds hold
independently of the input dimension $s$. 
This approach aligns with some realistic applications, such as parametric PDEs in uncertainty quantification, where explicit modeling of input random fields is possible.

The numerical experiments in \cite{KKNS25a} as well as new experiments here for the swish and ReLU activation functions show that our tailored regularization consistently performs better than the standard $\ell_2$ regularization.

Given that DNNs are nonlinear algorithms for function approximation, it is natural to ask how our lattice-based periodic DNN \eqref{KKNS:eq:DNN-per} compares to the lattice-based approximation algorithms $A^\dagger$ and $A^*$ (see \eqref{KKNS:eq:trig} and \eqref{KKNS:eq:ker}). For the parametric PDE problem from~\cite{KKKNS22}, we know that the algorithm $A^*$ (and also $A^\dagger$) gives the $L_2$ approximation error convergence rate $N^{-r/2}$ with $r = 1/p^*-1/2$ under the periodic function space setting (b) with norm \eqref{KKNS:eq:norm-per}. On the other hand, Theorem~\ref{KKNS:thm:err}(b)(c) for our lattice-based periodic DNN achieves $N^{-r/2}$ with $r = 1/p^*-1/2$ or $r=1/p^*$, but with an additional training error ${\tt tol}$ to balance. An interesting direction for future research would be to compare the performance of our lattice-based periodic DNN with these methods in practical applications to the same PDE problem discussed in~\cite{KKKNS22}.

\section*{Acknowledgments}

We acknowledge the financial support from the Australian Research Council (DP240100769) and the Research Foundation Flanders (FWO G091920N).

% ------------------------------------------------------------------------------------------
% ------------------------------------------------------------------------------------------
% ------------------------------------------------------------------------------------------
% ------------------------------------------------------------------------------------------
% ------------------------------------------------------------------------------------------

\appendix
\section*{Appendix}

\noindent\textbf{Proof of Theorem~\ref{KKNS:thm:err}(a).}
The proof follows closely the argument in \cite{KSS12}. Starting from \eqref{KKNS:eq:start}, for each component $1\le p\le N_{\rm obs}$, with tailored randomly-shifted lattice rules we have for all $\lambda\in (\tfrac{1}{2},1]$,
\begin{align*}
  {\rm r.m.s.}\; e^{\rm wor}_N(\bsz)\,
  \big\| \big(G(\cdot)_p - G_{\theta}^{[L]}(\cdot)_p\big)^2 \big\|_{\calW_{1,\bsgamma}} 
 \le \bigg(\frac{2\,(16\,C^4)^\lambda\,C_{s,\bsgamma,\lambda}}{N}\bigg)^{\frac{1}{2\lambda}},
\end{align*}
with
\begin{align*} 
 C_{s,\bsgamma,\lambda} := 
 \bigg(
 \sum_{\emptyset\ne\setu\subseteq\{1:s\}} \gamma_\setu^\lambda\,
 \bigg[\frac{2\zeta(2\lambda)\,2^\lambda}{(2\pi)^{2\lambda}} \bigg]^{|\setu|}\,
 \bigg)
 \bigg(\sum_{\setu\subseteq\{1:s\}} \frac{1}{\gamma_\setu} \bigg(
 (|\setu|+1)!\, \prod_{j\in\setu} (\kappa\,S_L\,b_j)
 \bigg)^2
 \bigg)^\lambda,
\end{align*}
where we used the root-mean-square worst case error bound \eqref{KKNS:eq:wce-np} and the norm bound \eqref{KKNS:eq:final-np} but with $b_j$ replaced by $\kappa\,S_L\,b_j$, where $\kappa$ is as in \eqref{KKNS:eq:kappa}. 

The special case $\kappa\,S_L = 1$ corresponds to the ideal scenario when \eqref{KKNS:eq:demand2} is true. In this case, the choice of weights $\gamma_\setu$ in \eqref{KKNS:eq:weight-np} minimizes $C_{s,\bsgamma,\lambda}$ and equates the terms inside the two sums. This is analogous to \cite[Theorem~6.4]{KSS12}.

For all $\kappa\,S_L \ge 1$, we now show that the choice of weights \eqref{KKNS:eq:weight-np} ensures that $C_{s,\bsgamma,\lambda}$ is bounded independently of $s$. Substituting \eqref{KKNS:eq:weight-np} into $C_{s,\bsgamma,\lambda}$, we obtain
\begin{align*}
 C_{s,\bsgamma,\lambda}
 &=
 \bigg(
 \sum_{\emptyset\ne\setu\subseteq\{1:s\}} [(|\setu|+1)!]^{\frac{2\lambda}{1+\lambda}}
 \,\bigg[\frac{2\zeta(2\lambda)\,2^\lambda}{(2\pi)^{2\lambda}} \bigg]^{\frac{\lambda\,|\setu|}{1+\lambda}}
 \prod_{j\in\setu} b_j^{\frac{2\lambda}{1+\lambda}}
 \bigg) \\
 &\qquad\times
 \bigg(
 \sum_{\setu\subseteq\{1:s\}} [(|\setu|+1)!]^{\frac{2\lambda}{1+\lambda}}
 \,\bigg[\frac{2\zeta(2\lambda)\,2^\lambda}{(2\pi)^{2\lambda}} \bigg]^{\frac{\lambda\,|\setu|}{1+\lambda}}
 (\kappa\,S_L)^{2|\setu|}
 \prod_{j\in\setu} b_j^{\frac{2\lambda}{1+\lambda}}
 \bigg)^\lambda  \\
 &\le  
 \bigg(
 \sum_{|\setu|<\infty} [(|\setu|+1)!]^{\frac{2\lambda}{1+\lambda}}
 \prod_{j\in\setu} (\varrho\, b_j)^{\frac{2\lambda}{1+\lambda}} 
 \bigg)^{1+\lambda},
 \qquad \varrho := 
 \bigg[\frac{2\zeta(2\lambda)\,2^\lambda}{(2\pi)^{2\lambda}} \bigg]^{\frac{1}{2}}
 (\kappa\,S_L)^{\frac{1+\lambda}{\lambda}}.
\end{align*}
Now consider arbitrary $\alpha_j>0$ to be specified later. Since $\lambda \in (1/2,1)$ and thus $\frac{2\lambda}{1+\lambda}<1$, we can apply H\"older's inequality with conjugate exponents $\frac{1+\lambda}{2\lambda}$ and $\frac{1+\lambda}{1-\lambda}$ as follows
\begin{align*}
 C_{s,\bsgamma,\lambda}^{\frac{1}{1+\lambda}}
 &\le  \sum_{|\setu|<\infty} [(|\setu|+1)!]^{\frac{2\lambda}{1+\lambda}}
 \prod_{j\in\setu} \alpha_j^{\frac{2\lambda}{1+\lambda}}
 \prod_{j\in\setu} \bigg(\frac{\varrho\, b_j}{\alpha_j}\bigg)^{\frac{2\lambda}{1+\lambda}} \\
 &\le  
 \bigg(\sum_{|\setu|<\infty} (|\setu|+1)! \prod_{j\in\setu} \alpha_j
 \bigg)^{\frac{2\lambda}{1+\lambda}}
 \bigg(\sum_{|\setu|<\infty} 
 \prod_{j\in\setu} \bigg(\frac{\varrho\, b_j}{\alpha_j}\bigg)^{\frac{2\lambda}{1-\lambda}}
 \bigg)^{\frac{1-\lambda}{1+\lambda}} \\
 &\le  
 \bigg(\frac{1}{1- \sum_{j\ge 1} \alpha_j}\bigg)^{\frac{4\lambda}{1+\lambda}}
 \exp\bigg(\frac{1-\lambda}{1+\lambda}\, \varrho^{\frac{2\lambda}{1-\lambda}}
 \sum_{j\ge 1} \bigg(\frac{b_j}{\alpha_j}\bigg)^{\frac{2\lambda}{1-\lambda}}
 \bigg),
\end{align*}
where we used $\sum_{|\setu|<\infty} (|\setu|+1)! \prod_{j\in\setu} \alpha_j \le (\frac{1}{1- \sum_{j\ge 1} \alpha_j})^2$ and $\sum_{|\setu|<\infty} \prod_{j\in\setu} \beta_j \le \exp(\sum_{j\ge 1} \beta_j)$, see \cite[Lemma~6.3]{KSS12}, which hold provided that 
\[
  \sum_{j\ge 1} \alpha_j < 1
  \quad\mbox{and}\quad
  \sum_{j\ge 1} \bigg(\frac{b_j}{\alpha_j}\bigg)^{\frac{2\lambda}{1-\lambda}} < \infty.
\]
Given that $\sum_{j\ge 1} b_j^{p^*} < \infty$, we now choose $\alpha_j = b_j^{p^*}/\vartheta$ for some $\vartheta > \sum_{j\ge 1} b_j^{p^*}$. Then we have $\sum_{j\ge 1} \alpha_j < 1$ and for the second sum to be finite we need
\[
  \frac{2\lambda}{1-\lambda} (1-p^*) \ge p^* \iff \lambda\ge \frac{p^*}{2-p^*}.
\]
We want $\lambda\in (1/2,1)$ to be as small as possible to get the best convergence rate. Hence we specify 
$\lambda$ as in \eqref{KKNS:eq:weight-np} to ensure the best convergence rate with implied constant independent of dimension $s$. This completes the proof.
\hfill\qed

\bigskip

\noindent\textbf{Proof of Theorem~\ref{KKNS:thm:err}(b).}
The proof follows closely the argument in \cite{KKKNS22}. Starting from \eqref{KKNS:eq:start}, for each component $1\le p\le N_{\rm obs}$, with tailored lattice rules we have for all $\lambda\in (\tfrac{1}{2\alpha},1]$,
\begin{align*}
 e^{\rm wor}_N(\bsz)\,
  \big\| \big(G(\cdot)_p - G_{\theta}^{[L]}(\cdot)_p\big)^2 \big\|_{\calW_{\alpha,\bsgamma}} 
 \le \bigg(\frac{2\,(16\,C^4)^\lambda\,C_{s,\bsgamma,\lambda}}{N}\bigg)^{\frac{1}{2\lambda}},
\end{align*}
with
\begin{align*} 
 C_{s,\bsgamma,\lambda} &:= 
 \bigg(
 \sum_{\emptyset\ne\setu\subseteq\{1:s\}} \!\!\! \gamma_\setu^\lambda\,
 \bigg[\frac{2\zeta(2\alpha\lambda)}{(2\pi)^{2\alpha\lambda}} \bigg]^{|\setu|}
 \bigg) \\
 &\qquad \times
 \bigg(\sum_{\setu\subseteq\{1:s\}}\!\!\!  \frac{(2\pi)^{2\alpha|\setu|}}{\gamma_\setu} \bigg(
 \sum_{\bsm_\setu\le\bsalpha_\setu} (|\bsm_\setu|+1)!\,
 \calS(\bsalpha_\setu,\bsm_\setu)\, \prod_{j\in\setu} (\kappa\,S_L\,b_j)^{m_j}\!
 \bigg)^2
 \bigg)^\lambda,
\end{align*}
where $\calS(\alpha_\setu,\bsm_\setu) := \prod_{j\in\setu} \calS(\alpha_,m_j)$ and
we used the worst case error bound \eqref{KKNS:eq:wce-per} and the norm bound \eqref{KKNS:eq:final-per} but with $b_j$ replaced by $\kappa\,S_L\,b_j$, where $\kappa$ is as in \eqref{KKNS:eq:kappa}. As in the previous proof, we proceed to show that the choice of weights \eqref{KKNS:eq:weight-per} allows $C_{s,\bsgamma,\lambda}$ to be bounded independently of $s$.

Abbreviating the choice of weights \eqref{KKNS:eq:weight-per} by
\[
  \gamma_\setu = \varrho^{|\setu|}
  \sum_{\bsm_\setu\le\bsalpha_\setu} [V_\setu(\bsm_\setu)]^{\frac{2}{1+\lambda}},
  \quad
  \varrho := \frac{(2\pi)^{2\alpha}}{[2\zeta(2\alpha\lambda)]^{\frac{1}{1+\lambda}}},
\]
\[
  V_\setu(\bsm_\setu) := (|\bsm_\setu|+1)!\,
  \calS(\bsalpha_\setu,\bsm_\setu)\, \prod_{j\in\setu} b_j^{m_j},
\]
and substituting it into $C_{s,\bsgamma,\lambda}$, we obtain
\begin{align*} 
 &C_{s,\bsgamma,\lambda}
 =
 \bigg(
 \sum_{\emptyset\ne\setu\subseteq\{1:s\}} 
 \bigg(\varrho^{|\setu|} \sum_{\bsm_\setu\le\bsalpha_\setu} [V_\setu(\bsm_\setu)]^{\frac{2}{1+\lambda}}\bigg)^\lambda\,
 \bigg[\frac{2\zeta(2\alpha\lambda)}{(2\pi)^{2\alpha\lambda}} \bigg]^{|\setu|}\,
 \bigg) \\
 &\;\times
 \bigg(\sum_{\setu\subseteq\{1:s\}}  \frac{(2\pi)^{2\alpha|\setu|}}
 {\varrho^{|\setu|}\sum_{\bsm_\setu\le\bsalpha_\setu} [V_\setu(\bsm_\setu)]^{\frac{2}{1+\lambda}}} 
 \bigg(
 \sum_{\bsm_\setu\le\bsalpha_\setu} [V_\setu(\bsm_\setu)]^{\frac{1}{1+\lambda}}\,
 [V_\setu(\bsm_\setu)]^{\frac{\lambda}{1+\lambda}}\,
 (\kappa\,S_L)^{|\bsm_\setu|}
 \bigg)^2
 \bigg)^\lambda \\
 &\le
 \bigg(
 \sum_{\setu\subseteq\{1:s\}} 
 [2\zeta(2\alpha\lambda)]^{\frac{|\setu|}{1+\lambda}}
 \sum_{\bsm_\setu\le\bsalpha_\setu} [V_\setu(\bsm_\setu)]^{\frac{2\lambda}{1+\lambda}}\,
 \bigg)\\
 &\qquad\times
 \bigg(\sum_{\setu\subseteq\{1:s\}}  [2\zeta(2\alpha\lambda)]^{\frac{|\setu|}{1+\lambda}}
 \sum_{\bsm_\setu\le\bsalpha_\setu} [V_\setu(\bsm_\setu)]^{\frac{2\lambda}{1+\lambda}}
 (\kappa\,S_L)^{2|\bsm_\setu|}
 \bigg)^\lambda  \\
 &\le
 \bigg(\sum_{\setu\subseteq\{1:s\}}  [2\zeta(2\alpha\lambda)]^{\frac{|\setu|}{1+\lambda}}
 \sum_{\bsone_\setu\le\bsm_\setu\le\bsalpha_\setu}
 [V_\setu(\bsm_\setu)]^{\frac{2\lambda}{1+\lambda}} 
 (\kappa\,S_L)^{2|\bsm_\setu|}
 \bigg)^{1+\lambda}  \\
 &\le
 \bigg(\sum_{\bsm\le\bsalpha}  
 \bigg((|\bsm|+1)!\,
  \prod_{j=1}^s \widetilde{b}_j^{m_j}
  \bigg)^{\frac{2\lambda}{1+\lambda}}
 \bigg)^{1+\lambda},
 \quad \widetilde{b}_j := \calS_{\max}(\alpha)\,[2\zeta(2\alpha\lambda)]^{\frac{1}{2\lambda}}\,
 (\kappa\,S_L)^{\frac{1+\lambda}{\lambda}}\,
 b_j,
\end{align*}
with $\calS_{\max}(\alpha):= \max_{1\le m\le \alpha} \calS(\alpha,m)$. In the first inequality above, for the first factor we dropped the condition $\setu\ne\emptyset$ and applied Jensen's inequality (i.e., $(\sum_i a_i)^\lambda \le \sum_i a_i^\lambda$ for $a_i\ge 0$ and $0<\lambda\le 1$) on the sum over $\bsm_\setu$, while for the second factor we applied the Cauchy-Schwarz inequality on the inner sum over $\bsm_\setu$ and then cancelled out the sum in the denominator. In the second inequality above, we combined the two factors from the previous step using $\kappa\,S_L\ge 1$, and we also added explicitly the condition $\bsm_\setu\ge \bsone_\setu$, since if any $m_j=0$ then $\calS(\alpha,m_j) = 0$ and so $V_\setu(\bsm_\setu) = 0$. In the third inequality above, we used $\calS(\alpha,m_j)\le \calS_{\max}(\alpha) \le [\calS_{\max}(\alpha)]^{m_j}$ and $2\zeta(2\alpha\lambda)\le [2\zeta(2\alpha\lambda)]^{m_j}$ to bring these factors inside the definition of $\widetilde{b}_j$, and then we rewrote the double sums over $\setu$ and $\bsm_\setu\ge\bsone_\setu$ equivalently as a single sum over $\bsm\le\bsalpha = (\alpha,\ldots,\alpha)\in\bbN^s$.

It remains to show that the last upper bound on $C_{s,\bsgamma,\lambda}$ can be bounded independently of~$s$. We follow closely the argument in \cite[Page~60, Proof of Theorem~4.5]{KKKNS22}, noting that we have $(|\bsm|+1)!$ here instead of $|\bsm|!$ and a different scaling in the definition of $\widetilde{b}_j$ relative to $b_j$.

We define the sequence $d_j := \widetilde{b}_{\lceil j/\alpha\rceil}$ for $j\ge 1$, so that $d_1= \cdots=d_\alpha = \widetilde{b}_1$, $d_{\alpha+1}=\cdots=d_{2\alpha} = \widetilde{b}_2$, and so on. Then
\begin{align} \label{KKNS:eq:def-dj}
 \prod_{j=1}^s \widetilde{b}_j^{m_j}
 = \underbrace{\widetilde{b}_1\cdots \widetilde{b}_1}_{m_1\,\rm{factors}}\cdot
 \underbrace{\widetilde{b}_2\cdots \widetilde{b}_2}_{m_2\,\rm{factors}} \cdots
 \underbrace{\widetilde{b}_s\cdots \widetilde{b}_s}_{m_s\,\rm{factors}} 
 = \prod_{j\in\setv_\bsm} d_j,
\end{align}
where $\setv_\bsm := \{1,\ldots,m_1, \alpha+1,\ldots,\alpha+m_2, \ldots, (s-1)\alpha+1,\ldots,(s-1)\alpha+m_s\}$ and $|\setv_\bsm| = |\bsm|$. Thus
\begin{align} \label{KKNS:eq:ratio}
 &\sum_{\bsm\le\bsalpha}  
 \bigg((|\bsm|+1)!\,\prod_{j=1}^s \widetilde{b}_j^{m_j}
  \bigg)^{\frac{2\lambda}{1+\lambda}}
 \le \sum_{\substack{\setv\subset\bbN \\ |\setv|<\infty}} 
 \bigg((|\setv|+1)!\,\prod_{j\in\setv} d_j
  \bigg)^{\frac{2\lambda}{1+\lambda}} \nonumber\\
 &= \sum_{\ell\ge 0} [(\ell+1)!]^{\frac{2\lambda}{1+\lambda}} 
 \sum_{\substack{\setv\subset\bbN \\ |\setv|=\ell}} 
 \prod_{j\in\setv} d_j^{\frac{2\lambda}{1+\lambda}} 
 \le \sum_{\ell\ge 0} \underbrace{[(\ell+1)!]^{\frac{2\lambda}{1+\lambda}} 
 \frac{1}{\ell!} \bigg(
 \sum_{j\ge 1} d_j^{\frac{2\lambda}{1+\lambda}} \bigg)^\ell}_{=:\, A_\ell}.
\end{align}
The last inequality holds since $(\sum_{j\ge 1} d_j^{\frac{2\lambda}{1+\lambda}})^\ell$ includes all the products of the form $\prod_{j\in\setv} d_j^{\frac{2\lambda}{1+\lambda}}$ with $|\setv|=\ell$, and moreover, includes each such term $\ell!$ times.

Recall that $\sum_{j\ge 1} b_j^{p^*}<\infty$. Our choice of $\lambda := \frac{p^*}{2-p^*}$ in \eqref{KKNS:eq:weight-per} gives $\frac{2\lambda}{1+\lambda} = p^*$, and thus
\begin{align*}
 \sum_{j\ge 1} d_j^{\frac{2\lambda}{1+\lambda}}
 = \sum_{j\ge 1} d_j^{p^*}
 = \alpha\sum_{j\ge 1} \widetilde{b}_j^{p^*}
 = \alpha\,
 [\calS_{\max}(\alpha)]^{p^*}\,[2\zeta(2\alpha\lambda)]^{\frac{1}{1+\lambda}}\,
 (\kappa\,S_L)^2
 \sum_{j\ge 1} b_j^{p^*} < \infty,
\end{align*}
provided that $2\alpha\lambda>1$, which is equivalent to $\alpha > \frac{1}{p^*}-\frac{1}{2}$. This justifies our choice for the integer value of $\alpha := \lfloor \frac{1}{p^*}-\frac{1}{2}+1 \rfloor = \lfloor \frac{1}{p^*}+\frac{1}{2} \rfloor$.

Finally, we claim that the last sum over $\ell$ in \eqref{KKNS:eq:ratio} is finite by the ratio test, since
\begin{align*}
  \lim_{\ell\to\infty} \frac{A_{\ell+1}}{A_\ell}
  = \lim_{\ell\to\infty} \frac{(\ell+2)^{p^*}}{\ell+1} 
  \sum_{j\ge 1} d_j^{p^*}
  \le \lim_{\ell\to\infty} \frac{(\ell+1)^{p^*} + 1}{\ell+1} 
  \sum_{j\ge 1} d_j^{p^*}
  = 0.
\end{align*}
In turn, we conclude that $C_{s,\gamma,\lambda}$ is bounded independently of $s$.
\hfill\qed

\bigskip

\noindent\textbf{Proof of Theorem~\ref{KKNS:thm:err}(c).}
The proof follows closely the argument in \cite{KKS20}. Starting from \eqref{KKNS:eq:start}, for each component $1\le p\le N_{\rm obs}$, with tailored lattice rules we have for all $\lambda\in (\tfrac{1}{\alpha},1]$,
\begin{align*}
 e^{\rm wor}_N(\bsz)\,
  \big\| \big(G(\cdot)_p - G_{\theta}^{[L]}(\cdot)_p\big)^2 \big\|_{\calW_{\alpha,\bsgamma}} 
 \le \bigg(\frac{2\,(4\,C^2)^\lambda\,C_{s,\bsgamma,\lambda}}{N}\bigg)^{\frac{1}{\lambda}},
\end{align*}
with
\begin{align*} 
 C_{s,\bsgamma,\lambda} &:= 
 \bigg(
 \sum_{\emptyset\ne\setu\subseteq\{1:s\}} \!\!\! \gamma_\setu^\lambda\,
 \bigg[\frac{2\zeta(\alpha\lambda)}{(2\pi)^{\alpha\lambda}} \bigg]^{|\setu|}
 \bigg) \\
 &\qquad \times
 \bigg(\max_{\setu\subseteq\{1:s\}}\!\!\!  \frac{(2\pi)^{\alpha|\setu|}}{\gamma_\setu} 
 \sum_{\bsm_\setu\le\bsalpha_\setu} (|\bsm_\setu|+1)!\,
 \calS(\bsalpha_\setu,\bsm_\setu)\, \prod_{j\in\setu} (\kappa\,S_L\,b_j)^{m_j}\!
 \bigg)^\lambda,
\end{align*}
where we used the worst case error bound \eqref{KKNS:eq:wce-K} and the norm bound \eqref{KKNS:eq:final-K} but with $b_j$ replaced by $\kappa\,S_L\,b_j$, where $\kappa$ is as in \eqref{KKNS:eq:kappa}. 

Differently to Theorem~\ref{KKNS:thm:err}(a) and Theorem~\ref{KKNS:thm:err}(b), the weights $\gamma_\setu$ are now chosen such that, in the case of $\kappa\,S_L=1$, 
the maximum over set $\setu$ is exactly $1$. Abbreviating the choice of weights \eqref{KKNS:eq:weight-K} by
\[
  \gamma_\setu = (2\pi)^{\alpha|\setu|}
  \sum_{\bsm_\setu\le\bsalpha_\setu} V_\setu(\bsm_\setu),
  \quad
  V_\setu(\bsm_\setu) := (|\bsm_\setu|+1)!\,
  \calS(\bsalpha_\setu,\bsm_\setu)\, \prod_{j\in\setu} b_j^{m_j},
\]
and substituting it into $C_{s,\bsgamma,\lambda}$, we obtain
\begin{align*} 
 C_{s,\bsgamma,\lambda}
 &=
 \bigg(
 \sum_{\emptyset\ne\setu\subseteq\{1:s\}} 
 \bigg( 
 \sum_{\bsm_\setu\le\bsalpha_\setu} V_\setu(\bsm_\setu)\bigg)^\lambda\,
 [ 2\zeta(\alpha\lambda) ]^{|\setu|}
 \bigg) \\
 &\quad\times
 \bigg(\max_{\setu\subseteq\{1:s\}}  \frac{1}
 {\sum_{\bsm_\setu\le\bsalpha_\setu} V_\setu(\bsm_\setu)} 
 \sum_{\bsm_\setu\le\bsalpha_\setu} 
 V_\setu(\bsm_\setu)\,(\kappa\,S_L)^{|\bsm_\setu|}
 \bigg)^\lambda.
\end{align*}

Note that in all sums over $\bsm_\setu$ we have $\bsm_\setu
\ge\bsone_\setu$, otherwise $V_\setu(\bsm_\setu) = 0$ due to $\calS(\alpha,0) = 0$. 
When $\kappa\,S_L = 1$, the maximum over $\setu$ is exactly $1$ so we can ignore the second factor.
Using Jensen's inequality, the sum over $\setu$ in the first factor can be bounded by
\begin{align*} 
 C_{s,\bsgamma,\lambda}
 \le \sum_{\setu\subseteq\{1:s\}} 
 \sum_{\bsone_\setu\le \bsm_\setu\le\bsalpha_\setu} [V_\setu(\bsm_\setu)]^\lambda\,
 [ 2\zeta(\alpha\lambda) ]^{|\setu|} 
 &\le \sum_{\bsm\le\bsalpha} \bigg( (|\bsm|+1)! \prod_{j=1}^s \widetilde{b}_j^{m_j}\bigg)^\lambda \\
 &\le \sum_{\ell\ge 0} [(\ell+1)!]^\lambda\,\frac{1}{\ell !} \bigg(\sum_{j\ge 1} d_j^\lambda\bigg)^\ell,
\end{align*}
where now $\widetilde{b}_j := \calS_{\max}(\alpha)\,[ 2\zeta(\alpha\lambda) ]^{\frac{1}{\lambda}}\,b_j$, and
$d_j$ is defined as in \eqref{KKNS:eq:def-dj}--\eqref{KKNS:eq:ratio} but using the new definition of $\widetilde{b}_j$. Taking now $\lambda = p^*$ yields 
\[
  \sum_{j\ge 1} d_j^\lambda = \sum_{j\ge 1} d_j^{p^*}
  = \alpha \sum_{j\ge 1} \widetilde{b}_j^{p^*}
  = \alpha\,[\calS_{\max}(\alpha)]^{p^*}\,[ 2\zeta(\alpha\lambda)]\,  \sum_{j\ge 1} b_j^{p^*} < \infty,
\]
provided that $\alpha\lambda>1$, which justifies the choice $\alpha := \lfloor\frac{1}{p^*}\rfloor+1$ in \eqref{KKNS:eq:weight-K}. The same argument as in the last part of the proof of Theorem~\ref{KKNS:thm:err}(b) then implies that $C_{s,\bsgamma,\lambda}$ is bounded independently of $s$.

When $\kappa\,S_L > 1$, since $\bsm_\setu \ge\bsone_\setu$, every factor inside the maximum over $\setu$ is at least $(\kappa\,S_L)^{|\setu|}$ and at most $(\kappa\,S_L)^{\alpha\,|\setu|}$. Hence $C_{s,\bsgamma,\lambda}$ includes this extra factor between $(\kappa\,S_L)^s$ and $(\kappa\,S_L)^{\alpha\,s}$, 
which unfortunately grows exponentially with $s$.
\hfill\qed

%%% To ensure the bibliography has the correct style please run bibtex with the spmpsci style
%%% which is included in the Springer zip file
%%% Here we assume refs.bib would be the name of the bib file containing bibliographic info
%%% You can then copy the .bbl produced, as given in the example below
%%%
%\bibliographystyle{spmpsci}
%\bibliography{refs}

\ifdefined\arxivstyle
 %\newpage
 \subsection*{Authors' addresses} 

 \textbf{Alexander Keller} \\
 NVIDIA, Berlin, Germany. Email: akeller@nvidia.com 

 \smallskip\noindent\textbf{Frances Y. Kuo} \\ 
 School of Mathematics and Statistics, UNSW Sydney, Sydney, Australia. Email: f.kuo@unsw.edu.au 

 \smallskip\noindent\textbf{Dirk Nuyens} \\ 
 Department of Computer Science, KU Leuven, Leuven, Belgium. Email: dirk.nuyens@kuleuven.be 

 \smallskip\noindent\textbf{Ian H. Sloan} \\ 
 School of Mathematics and Statistics, UNSW Sydney, Sydney, Australia. Email: i.sloan@unsw.edu.au
\fi

\end{document}